%% file: main.tex
\documentclass{article}

\usepackage{microtype}
\usepackage{graphicx}
\usepackage{booktabs}
\usepackage{hyperref}
\usepackage{color}

\usepackage[accepted]{icml2018}

\usepackage{graphicx,subcaption}
\usepackage{amsmath,amssymb,amsthm}
\usepackage{enumerate}
\usepackage{multirow}
\usepackage{natbib}
\usepackage{algorithm,algorithmic}
\theoremstyle{definition}
\newtheorem{theorem}{Theorem}[section]
\newtheorem{lemma}[theorem]{Lemma}
\newtheorem{claim}{Claim}
\newtheorem{definition}[theorem]{Definition}
\newtheorem{remark}[theorem]{Remark}

\DeclareMathOperator*{\diag}{diag}

\DeclareMathOperator*{\tr}{Tr}

\renewcommand{\r}{{\textrm r}}
\newcommand{\x}{{\textrm x}}
\newcommand{\y}{{\textrm y}}
\renewcommand{\v}{{\textrm v}}
\newcommand{\e}{{\textrm e}}
\renewcommand{\u}{{\textrm u}}
\renewcommand{\a}{{\textrm a}}

\newcommand{\0}{{\textrm 0}}
\newcommand{\1}{\textrm 1}
\newcommand{\n}{{\textrm n}}
\newcommand{\w}{{\textrm w}}
\renewcommand{\b}{{\textrm b}}

\newcommand{\h}{{\textrm h}}
\newcommand{\I}{\textrm{I}}
\newcommand{\A}{\textrm{A}}
\newcommand{\M}{\textrm{M}}
\newcommand{\G}{\textrm{G}}

\newcommand{\W}{\textrm{W}}

\newcommand{\D}{\textrm{D}}
\newcommand{\U}{\textrm{U}}
\newcommand{\V}{\textrm{V}}

\newcommand{\Rr}{\textrm{Q}}
\newcommand{\X}{\textrm{X}}
\newcommand{\Y}{\textrm{Y}}
\newcommand{\Zz}{\textrm{Z}}

\newcommand{\cB}{\mathcal{B}}
\newcommand{\cD}{\mathcal{D}}
\newcommand{\cE}{\mathcal{E}}

\newcommand{\cX}{\mathcal{X}}
\newcommand{\cS}{\mathcal{S}}
\newcommand{\cU}{\mathcal{U}}

\newcommand{\bb}{\mathbb}
\newcommand{\R}{\bb R}
\newcommand{\E}{\bb E}

\newcommand{\N}{{\bb N}}

\newcommand{\minim}[2]{
\underset{#1}{\textrm{min}} \ #2}

\usepackage{flushend}

\icmltitlerunning{On the Implicit Bias of Dropout}

\begin{document}

\twocolumn[
\icmltitle{On the Implicit Bias of Dropout}

\icmlsetsymbol{equal}{*}

\begin{icmlauthorlist}
\icmlauthor{Poorya Mianjy}{csjhu}
\icmlauthor{Raman Arora}{csjhu}
\icmlauthor{Rene Vidal}{bmejhu}
\end{icmlauthorlist}

\icmlaffiliation{csjhu}{Department of Computer Science, Johns Hopkins University, Baltimore, USA}
\icmlaffiliation{bmejhu}{Department of Biomedical Engineering, Johns Hopkins University, Baltimore, USA}

\icmlcorrespondingauthor{Raman Arora}{arora@cs.jhu.edu}

\icmlkeywords{Deep Learning, Dropout, Implicit Bias}

\vskip 0.3in
]

\printAffiliationsAndNotice{}

\begin{abstract}
Algorithmic approaches endow deep learning systems with implicit bias that helps them generalize even in over-parametrized settings. In this paper, we focus on understanding such a bias induced in learning through dropout, a popular technique to avoid overfitting in deep learning. For single hidden-layer linear neural networks, we show that dropout tends to make the norm of incoming/outgoing weight vectors of all the hidden nodes equal. In addition, we provide a complete characterization of the optimization landscape induced by dropout.
\end{abstract}

\input{introduction}
\input{symmetric}
\input{asymmetric}

\input{convergence}
\input{factorization}
\input{experiments}
\input{discussion}
\newpage

\section*{Acknowledgements}
\vspace*{-5pt}
\noindent This research was supported in part by NSF BIGDATA grant IIS-1546482 and NSF grant IIS-1618485. 
\bibliographystyle{plainnat}

\vspace*{-5pt}
\bibliography{references}
\clearpage

\input{appendix}

\end{document}

%% file: introduction.tex
\section{Introduction}
Modern machine learning systems based on deep neural networks are usually over-parameterized, i.e. the number of parameters in the model is much larger than the size of the training data, which makes these systems prone to overfitting. Several explicit regularization strategies have been used~in practice to help these systems generalize, including $\ell_1$ and~$\ell_2$ regularization of the parameters \citep{nowlan1992simplifying}. Recently, \cite{neyshabur2015norm} showed that a variety of such norm-based regularizers can provide size-independent capacity control, suggesting that the network size is not a good measure of complexity in such settings. Such a view had been previously motivated in the context of matrix factorization \citep{srebro2005maximum}, where it is preferable to have many factors of limited overall influence rather than a few important ones.

Besides explicit regularization techniques, practitioners have used a spectrum of algorithmic approaches to improve the generalization ability of  over-parametrized models. This includes early stopping of back-propagation~\citep{caruana2001overfitting}, batch normalization \citep{ioffe2015batch} and  dropout \citep{srivastava2014dropout}. In particular, dropout, which is the focus of this paper, randomly drops hidden nodes along with their connections at  training time. Dropout was introduced by \citet{srivastava2014dropout} as a way of breaking up co-adaptation among neurons, drawing insights from the success of the sexual reproduction model in the evolution of advanced organisms. While dropout has enjoyed tremendous success in training deep neural networks, the theoretical understanding of how dropout (and other algorithmic heuristics)  provide  
regularization in deep learning remains somewhat limited.

We argue that a prerequisite for understanding implicit regularization due to various  algorithmic heuristics in deep learning, including dropout, is to analyze their behavior in simpler models. Therefore, in this paper, we consider the following learning problem. Let $\x \in \R^{d_2}$ represent an  input feature vector with some unknown distribution $\cD$ such that $\E_{\x\sim\cD}[\x\x^\top]=\I$. The output label vector $\y \in \R^{d_1}$ is given as $\y = \M \x$ for some $\M \in \R^{d_1 \times d_2}.$ We consider the hypothesis class represented by a single hidden-layer linear network parametrized as $h_{\U,\V}(\x)=\U\V^\top \x$, where $\V\in \R^{d_2 \times r}$ and $\U\in \R^{d_1\times r}$ are the weight matrices in the first and  the second layers, respectively. The goal of  learning is to find weight matrices $\U, \V$ that minimize the expected~loss $\ell(\U, \V):=\E_{\x \sim \cD}[\|\y - h_{\U, \V}(\x)\|^2]=\E_{\x \sim \cD}[\|\y - \U\V^\top\x\|^2]$. 

A natural learning algorithm to consider is back-propagation with dropout, which can be seen as an instance of stochastic gradient descent on the following objective: 
\begin{equation}\label{eq:opt_dropout}
 \hspace*{-7pt}f(\U,\V)\!:=\!\E_{b_i \sim \text{Ber}(\theta),\x \sim \cD}\!\left[\left\| \y - \frac1{\theta} \U \diag(\b) \V^\top\x \right\|^2\right]\!, 
 \end{equation}
 where the expectation is w.r.t. the underlying distribution 
 on data as well as randomization due to dropout 
 (each hidden unit is dropped independently with probability $1-\theta$). This procedure, which we simply refer to as dropout in this paper, is given in Algorithm~\ref{alg:dropout}.

It is easy to check (see Lemma~\ref{lem:equiv} in the supplementary) that the objective in equation~(\ref{eq:opt_dropout}) can be written as 
\begin{align}\label{eq:main1}
f(\U,\V) = \ell(\U, \V) + \lambda \sum_{i=1}^{r}{\| \u_i \|^2 \| \v_i \|^2},
\end{align}
where $\lambda=\frac{1-\theta}{\theta}$ is the regularization parameter, and $\u_i$ and $\v_i$ represent the $i^{\text{th}}$ columns of $\U$ and $\V$, respectively. Note that while the goal was to minimize the expected squared loss, using dropout with gradient descent amounts to finding a minimum of the objective in equation~(\ref{eq:main1}); we argue that the additional term in the objective serves as a regularizer, 
$R(\U,\V):=\lambda\sum_{i=1}^{r}{\| \u_i \|^2 \|\v_i\|^2}$, and is an explicit instantiation of the implicit bias of dropout. Furthermore, we note that 
this regularizer is closely related to \textit{path regularization} which is given as the square-root of the sum over all paths, from input to output, of the product of the squared weights along the path~\cite{neyshabur2015norm}. Formally, for a single layer network, path regularization is given as
\begin{equation} 
\psi_2(\U ,\V) = \left( \sum_{i=1}^r \sum_{j=1}^{d_1} \sum_{k=1}^{d_2} u^2_{ji} v^2_{ki} \right)^\frac12.  
\end{equation}

Interestingly, the dropout regularizer is equal to the square of the path regularizer, i.e. $R(\U,\V)=\lambda \psi_2^2(\U, \V)$. While this observation is rather immediate, it has profound implications owing to the fact that path regularization provides size-independent capacity control in deep learning, thereby supporting empirical evidence that dropout finds good solutions in over-parametrized settings. 

In this paper, we focus on studying the optimization landscape of the objective in equation~(\ref{eq:main1}) for a single hidden-layer linear network with dropout and the special case of an autoencoder with tied weights. Furthermore, we are interested in  characterizing the solutions to which dropout (i.e. Algorithm~\ref{alg:dropout}) converges. We make the following progress toward addressing these questions. 
\begin{enumerate} 
\item We formally characterize the implicit bias of dropout. 
We show that, when minimizing the expected loss $\ell(\U,\V)$ with 
dropout, any global minimum $(\tilde\U,\tilde\V)$ satisfies $\psi_2(\tilde\U,\tilde\V)=\min\{\psi_2(\U,\V) \textrm{ s.t. } \U\V^\top=\tilde\U\tilde\V^\top\}$. More importantly, for auto-encoders with tied weights, we show that all \textit{local} minima inherit this property.

\item Despite the non-convex nature of the problem, we completely characterize the global optima by giving necessary and sufficient conditions for optimality.

\item We describe the optimization landscape of the dropout problem. In particular, we show that for a sufficiently small dropout rate, all local minima of the objective in equation~\eqref{eq:main1} are global and all saddle points are non-degenerate. This allows Algorithm~\ref{alg:dropout} to efficiently escape saddle points and converge to a global optimum. 
\end{enumerate}  
The rest of the paper is organized as follows. 
In Section~\ref{sec:sym}, we study dropout for single hidden-layer linear auto-encoder networks with weights tied between the first and the second layers. This gives us the tools to study the dropout problem in a more general setting of single hidden-layer linear networks in Section~\ref{sec:asym}. In Section~\ref{sec:landscape}, we characterize the optimization landscape of the objective in~(\ref{eq:main1}), show that it satisfies the strict saddle property, and that there are no spurious local minima. We specialize our results to matrix factorization in Section~\ref{sec:factorization}, and in Section~\ref{sec:experiments}, we discuss preliminary experiments to support our theoretical results.

\begin{algorithm}[t!]
\caption{\label{alg:dropout}Dropout with Stochastic Gradient Descent}
\begin{algorithmic}[1]
\INPUT Data $\{\!(\x_t, \y_t)\!\}_{t=0}^{T-1}$, dropout rate $1\!-\! \theta$, learning rate~$\eta$
\STATE Initialize $\U_0,\V_0$
\FOR{$t = 0,1,\dots,T-1$}
\STATE sample $\b_t$ element-wise from Bernoulli$(\theta)$
\STATE Update the weights
{\small
\begin{align*}
&\U_{t+1} \!\gets\! \U_{t} \!-\! \eta \! \left(\frac1{\theta}\U_{t}\diag(\b_t)\V_{t}^\top\x_t \!-\! \y_t \!\!\right) \! \x_t^\top \V_{t}\diag(\b_t)\\
&\V_{t+1} \!\gets\! \V_{t} \!-\! \eta \x_t \!\!\left(\frac1{\theta}\x_t^\top\V_{t}\diag(\b_t)\U_{t}^\top \!\!-\! \y_t^\top\!\!\right)\! \U_{t}\diag(\b_t)
\end{align*}
}
\ENDFOR
\OUTPUT $\U_T,\V_T$
\end{algorithmic}
\end{algorithm}

\subsection{Notation}\label{sec:notation}
We denote matrices, vectors, scalar variables and sets by Roman capital letters, Roman small letters, small letters and script letters respectively (e.g. $\X$, $\x$, $x$, and $\cX$). For any integer $d$, we represent the set $\{ 1,\ldots,d \}$ by $[d]$. For any integer $i$, $\e_i$ denotes the $i$-th standard basis. For any integer $d$, $\1_d \in \R^d$ is the vector of all ones, $\| \x \|$ represents the $\ell_2$-norm of vector $\x$, and $\| \X \|,\| \X \|_F$, $\| \X\|_*$ and $\lambda_i(\X)$ represent the spectral norm, the Frobenius norm, the nuclear norm and the $i$-th largest singular value of matrix $\X$, respectively. 
 $\langle \cdot, \cdot \rangle$ represents the standard inner product, for vectors or matrices, where $\langle \X,\X' \rangle = \tr(\X^\top \X')$. For a matrix $\X\in \R^{d_1\times d_2}$, $\diag(\X)\in \R^{\min\{ d_1,d_2 \}}$ returns its diagonal elements. Similarly, for a vector $\x\in\R^d$, $\diag(\x)\in \R^{d\times d}$ is a diagonal matrix with $\x$ on its diagonal.  For any scalar $x$, we define $(x)_+=\max\{x,0\}$, and for a matrix $\X$, $(\X)_+$ is the elementwise application of $(\cdot)_+$ to $\X$. For a matrix $\X$ with a compact singular value decomposition $\X=\U\Sigma\V^\top$, and for any scalar $\alpha \geq 0$, we define the singular-value shrinkage-thresholding operator as $\cS_\alpha(\X):=\U(\Sigma-\alpha\I)_+\V^\top$.
 

%% file: symmetric.tex
\section{Linear autoencoders with tied weights}\label{sec:sym}

We begin with a simpler hypothesis family of single hidden-layer linear auto-encoders with weights tied such that $\U=\V$. Studying the problem in this setting helps our intuition about the implicit bias that dropout induces on weight matrices $\U$. This analysis will be extended to the more general setting of single hidden-layer linear networks in the next section.

Recall that the goal here is to find an autoencoder network represented by a weight matrix $\U \in \R^{d_2 \times r}$ that solves: 
\begin{align}\label{eq:opt_sym}
\minim{\U\in\R^{d_2\times r}}{ \ell(\U, \U) + \lambda {\sum_{i=1}^{r}{\| \u_i \|^4}}}, 
\end{align}
where $\u_i$ is the $i^{\text{th}}$ column of $\U.$
Note that the loss function $\ell(\U,\U)$ is invariant under rotations, i.e., for any orthogonal transformation $\Rr\in \R^{d\times d}, \Rr^\top\Rr=\Rr\Rr^\top = \I_d$, it holds that $$\ell(\U,\!\U)\!=\! \E_{\x\sim \cD}[\|\y-\U\Rr\Rr^\top\U^\top\x\|^2]
\!=\!\ell(\U\Rr,\!\U\Rr),$$ so that applying a rotation matrix to a candidate solution $\U$ does not change the value of the loss function. However, the regularizer is not rotation-invariant and clearly depends on the choice of $\Rr$. Therefore, in order to solve Problem~(\ref{eq:opt_sym}), we need to find a rotation matrix that minimizes the value of the regularizer for a given weight matrix. 

To that end, let us denote the squared column norms of the weight matrix $\U$ by $\n_\u=(\| \u_1 \|^2,\ldots,\| \u_r\|^2)$ and let $\1_r \in \R^r$ be the vector of all ones. Then, for any $\U$, 
\begin{align*}
R(\U,\U)&=\lambda\sum_{i=1}^{r}{\| \u_i \|^4} =\frac{\lambda}{r} \| \1_r \|^2 \| \n_\u \|^2 \\
&\geq \frac{\lambda}{r} \langle \1_r,\n_\u \rangle^2 = \frac{\lambda}{r} \left( \sum_{i=1}^{r}{\| \u_i \|^2}\right)^2 =\frac{\lambda}{r} \| \U \|_F^4, 
\end{align*}
where the inequality follows from Cauchy-Schwartz inequality. Hence, the regularizer is lower bounded by $\frac{\lambda}{r}\| \U \|_F^4$, with equality if and only if $\n_\u$ is parallel to $\1_r$, i.e. when all the columns of $\U$ have equal norms. Since the loss function is rotation invariant, one can always decrease the value of the overall objective by rotating $\U$ such that $\U\Rr$ has a smaller regularizer. A natural question to ask, therefore, is 
 \textit{if there always exists a rotation matrix $\Rr$ such that the  matrix $\U\Rr$ has equal column norms.} In order to formally address this question, we introduce the following definition. 

\begin{figure*}[ht!]
\centering
\begin{tabular}{ccc}
$\lambda = 0$  & $\lambda = 0.6$  & $\lambda = 2$  \\ 
\hspace*{-22pt} 
\includegraphics[width=0.37\textwidth]{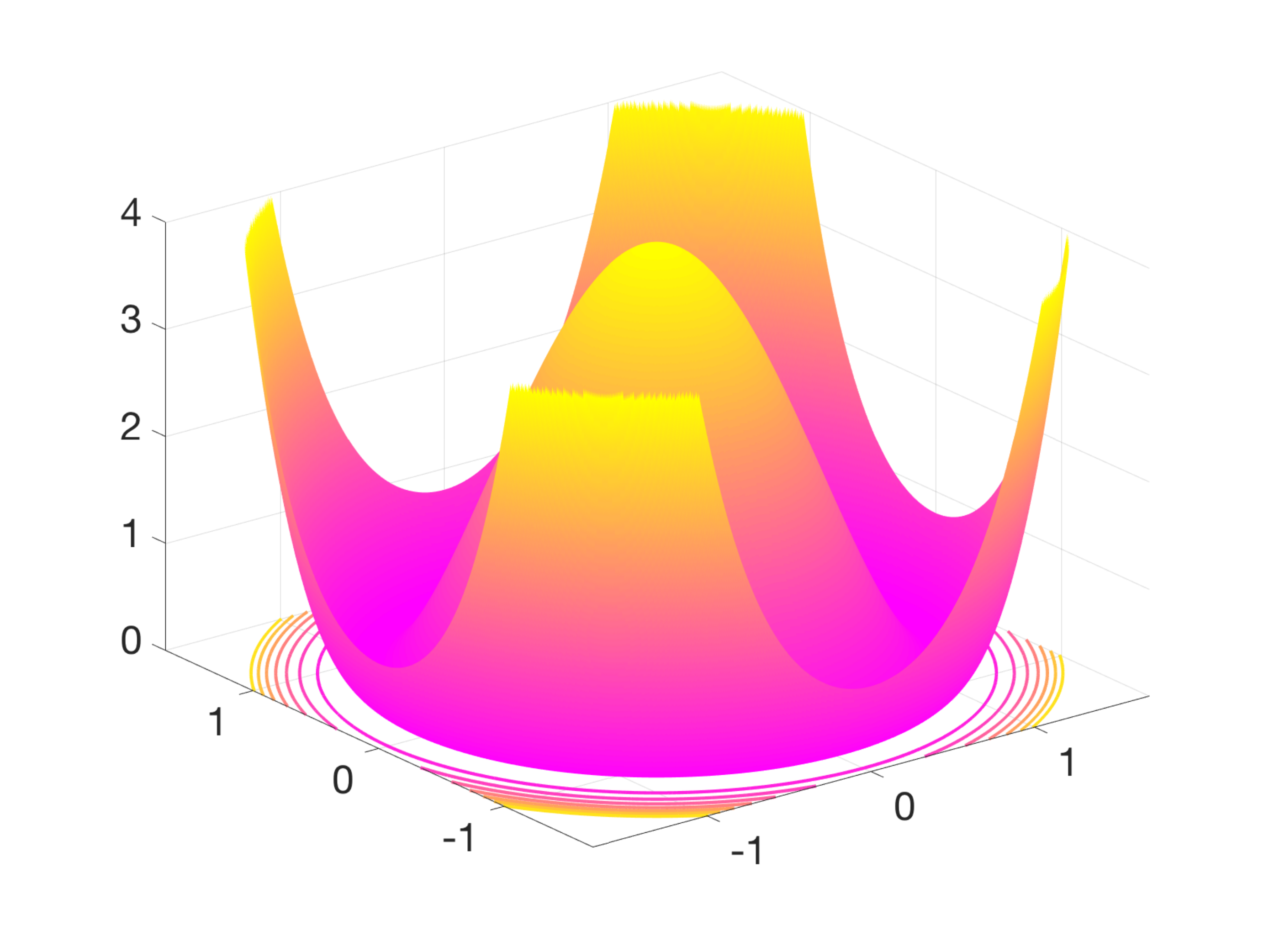}
&
\hspace*{-34pt} 
\includegraphics[width=0.37\textwidth]{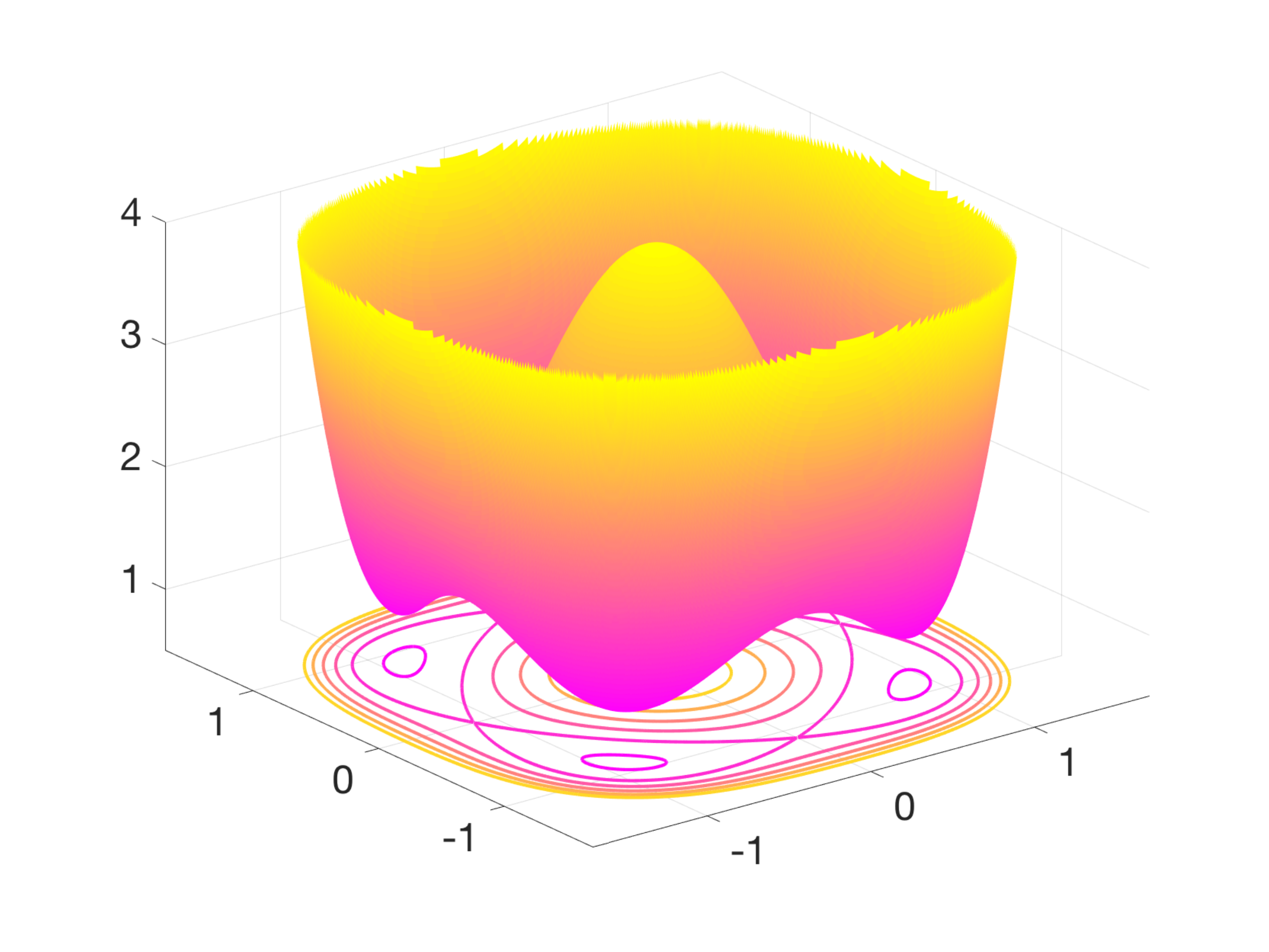}
&
\hspace*{-36pt} 
\includegraphics[width=0.37\textwidth]{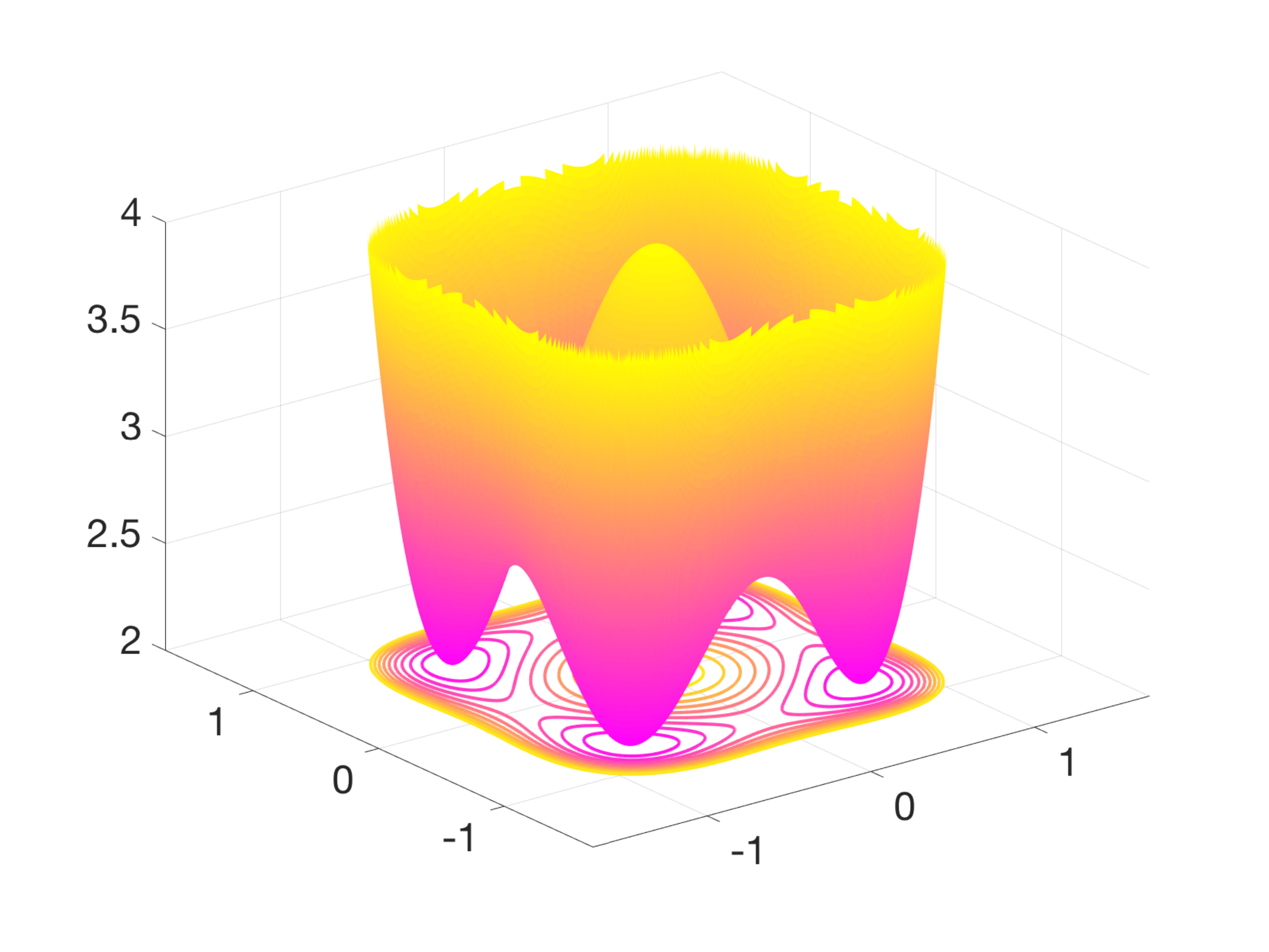} \\ 
\hspace*{-22pt} 
\includegraphics[width=0.37\textwidth]{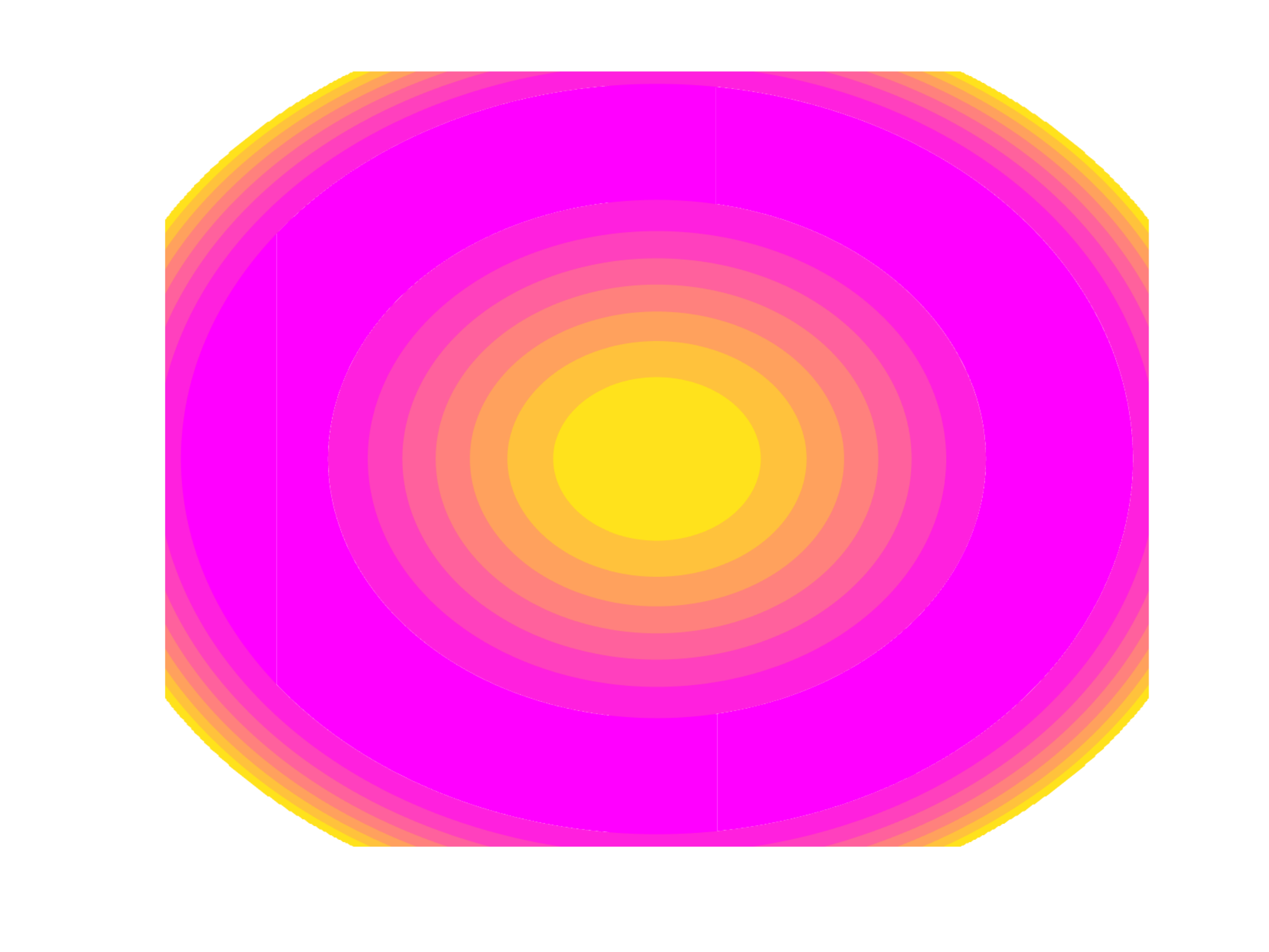}
&
\hspace*{-34pt} 
\includegraphics[width=0.37\textwidth]{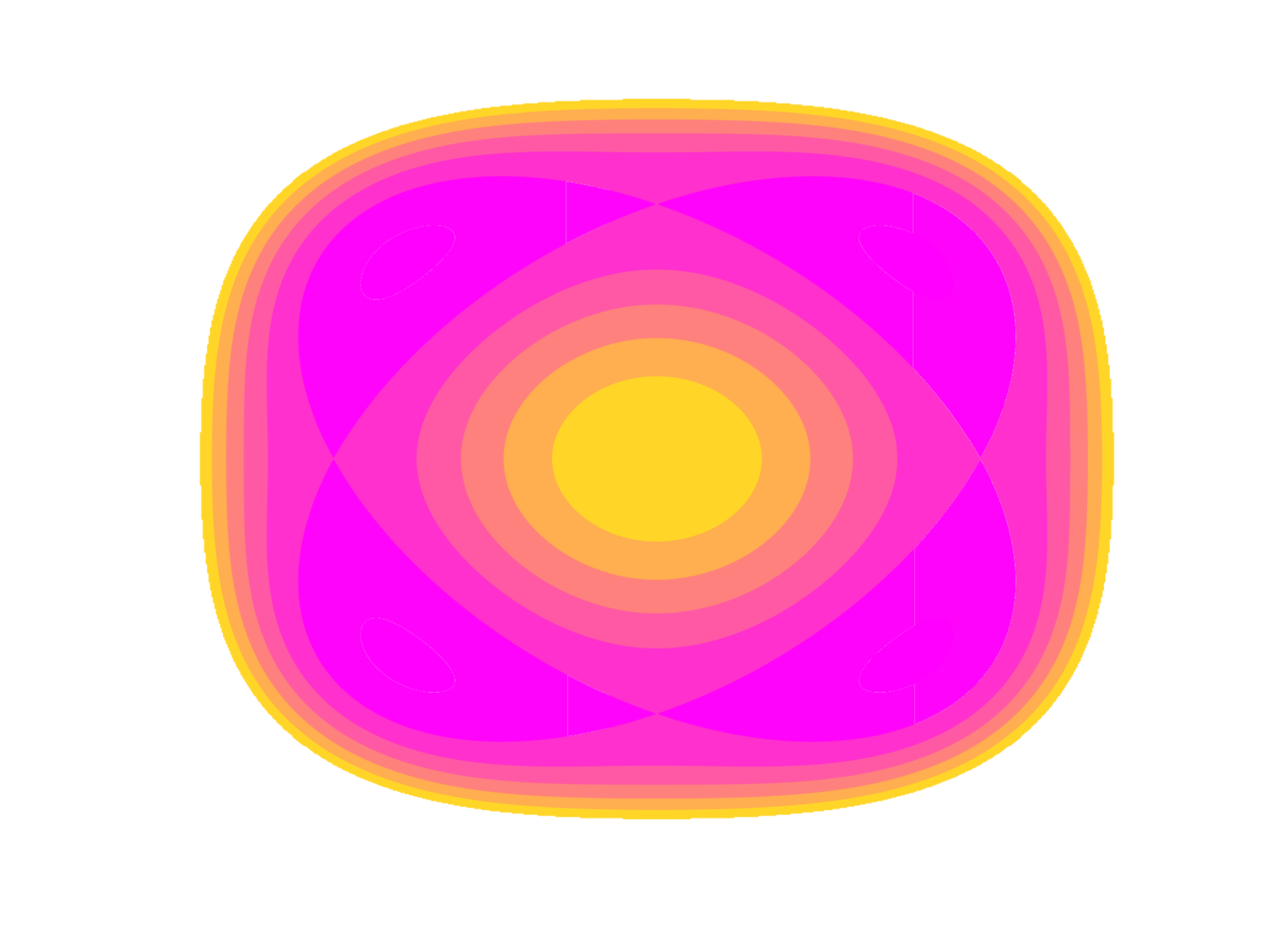}
&
\hspace*{-36pt} 
\includegraphics[width=0.37\textwidth]{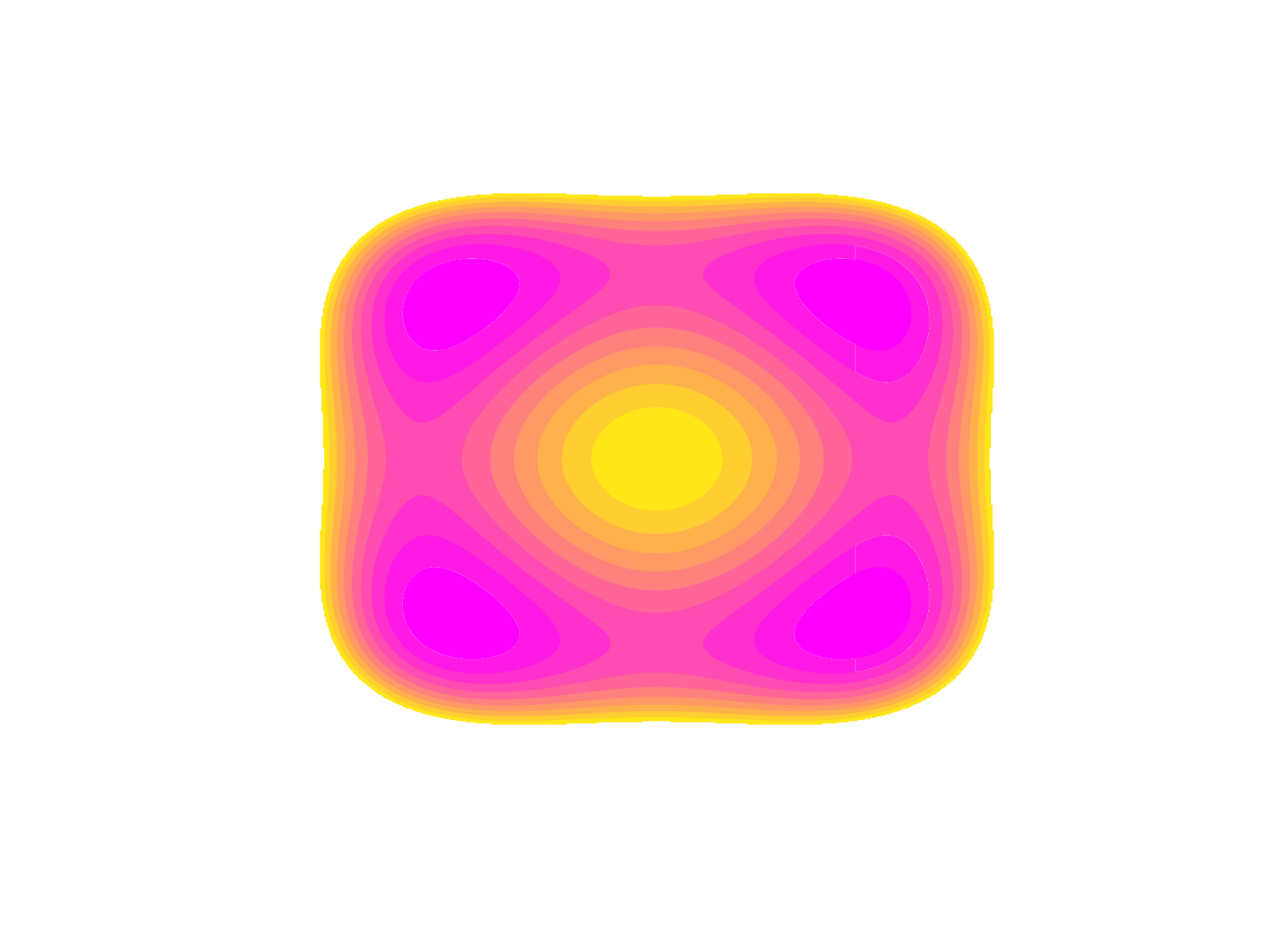}
\end{tabular}
\caption{\label{fig:autoenc}{Optimization landscape (top) and contour plot (bottom) for a single hidden-layer linear autoencoder network with one dimensional input and output and a hidden layer of width  $r=2$ with dropout, for different values of the  regularization parameter $\lambda$. Left: for $\lambda=0$ the problem reduces to squared loss minimization, which is rotation invariant as suggested by the level sets. Middle: for $\lambda > 0$ the global optima shrink toward the origin. All local minima are global, and are equalized, i.e. the weights are parallel to the vector $(\pm1,\pm1)$. Right: as $\lambda$ increases, global optima shrink further.}}
\end{figure*}

\begin{algorithm}[t]
\caption{\label{alg:equalizer}\texttt{EQZ}$(\U)$ equalizer of an auto-encoder $h_{\U,\U}$} 
\begin{algorithmic}[1]
\INPUT $\U \in \R^{d\times r}$ 
\STATE $\G \gets \U^\top \U$
\STATE $\Rr \gets \I_r$
\FOR{$i = 1$ to $r$}
\STATE $[\V, \Lambda] \!\gets\! \texttt{eig}(\G)$ \COMMENT{$\G\!=\!\V\Lambda\V^\top\!$~eigendecomposition}
\STATE $\w=\frac{1}{\sqrt{r-i+1}}\sum_{i=1}^{r-i+1}{\v_i}$
\STATE $\Rr_i \gets [\w \ \w_\perp]$ \COMMENT{$\w_\perp\in \R^{(r-i+1)\times(r-i)}$ orthonormal basis for the Null space of $\w$}
\STATE $\G \gets \Rr_i^\top \G \Rr_i$ \COMMENT{Making first diagonal element zero}
\STATE $\G \gets \G(2:\text{end},2:\text{end})$ \COMMENT{First principal submatrix}
\STATE $\Rr \gets \Rr \begin{bmatrix} &\I_{i-1} &\0 \\ &\0 &\Rr_i \end{bmatrix}$
\ENDFOR
\OUTPUT $\Rr$ \COMMENT{such that $\U\Rr$ is equalized}
\end{algorithmic}
\end{algorithm}

\begin{definition}[Equalized weight matrix, equalized autoencoder, equalizer]
A weight matrix $\U$ is said to be \textit{equalized} if all its columns have equal norms. An autoencoder with tied weights is said to be \textit{equalized} if the norm of the incoming weight vector is equal across all hidden nodes in the network. An orthogonal transformation $\Rr$ is said to be an \emph{equalizer} of $\U$ (equivalently, of the corresponding autoencoder) if $\U\Rr$ is equalized. \end{definition}

Next, we show that any matrix $\U$ can be equalized. 
\begin{theorem}\label{thm:equalization}
Any weight matrix $\U\in \R^{d\times r}$ (equivalently, the corresponding autoencoder network $h_{\U, \U}$) can be equalized. Furthermore, there exists a polynomial time algorithm (Algorithm~\ref{alg:equalizer}) that returns an equalizer for a given matrix.
\end{theorem}

The key insight here is that if $\G_\U:=\U^\top\U$ is the Gram matrix associated with the weight matrix $\U$, then $h_{\U,\U}$ is equalized by $\Rr$ if and only if all diagonal elements of $\Rr^\top\G_\U\Rr$ are equal. More importantly, if $\G_\U=\V\Lambda\V^\top$ is an eigendecomposition of $\G_\U$, then for $\w=\frac{1}{\sqrt r}\sum_{i=1}^{r}{\v_i}$, it holds that $\w^\top \G_\U \w = \frac{\tr \G_\U}{r}$; Proof of Theorem~\ref{thm:equalization} uses this property to recursively equalize all diagonal elements of $\G_\U$.

Finally, we argue that the implicit bias induced by dropout is closely related to the notion of equalized network introduced above. In particular, our main result of the section states that the dropout enforces any globally optimal network
to be equalized. Formally, we show the following. 
\begin{theorem}\label{thm:equalized_sym}
If $\U$ is a global optimum of Problem~\ref{eq:opt_sym}, then $\U$ is equalized. Furthermore, it holds that $$R(\U)=\frac{\lambda}{r}\| \U \|_F^4.$$
\end{theorem}

Theorem~\ref{thm:equalized_sym} characterizes the effect of regularization induced by dropout in learning autoencoders with tied weights. It states that for any globally optimal network, the columns of the corresponding weight matrix have equal norms. In other words, dropout tends to give equal weights to all hidden nodes -- it shows that dropout implicitly biases the optimal networks towards having hidden nodes with limited overall influence rather than a few important ones.

While Theorem~\ref{thm:equalized_sym} makes explicit the bias of dropout and gives a necessary condition for global optimality 
in terms of the weight matrix $\U_*$, it does not characterize the bias induced in terms of the network (i.e. in terms of $ \U_* \U_*^\top$). The following theorem completes the characterization by describing  globally optimal autoencoder networks. Since the goal is to understand the implicit bias of dropout, we specify the global optimum in terms of the true concept, $\M$.

\begin{theorem}\label{thm:sym_global}
For any $j\in[r]$, let  $\kappa_j:=\frac{1}{j}\sum_{i=1}^{j}{\lambda_i(\M)}$. Furthermore, define $\rho:=\max\{ j \in [r]: \ \lambda_j(\M) > \frac{\lambda j \kappa_j}{r+\lambda j} \}$. Then, if $\U_*$ is a global optimum of Problem~\ref{eq:opt_sym}, it satisfies that 
$\U_*\U_*^\top = \cS_{\frac{\lambda\rho\kappa_\rho}{r+\lambda\rho}}\left(\M\right)$. 
\end{theorem}
\begin{remark}
In light of Theorem~\ref{thm:equalized_sym}, the proof of Theorem~\ref{thm:sym_global} entails solving the following optimization problem
\vspace*{-10pt}
\begin{equation}\label{opt:lower_sym}
\minim{\U\in\R^{d\times r}}{\ell(\U, \U) + \frac{\lambda}{r}\| \U \|_F^4}, 
\end{equation}
instead of Problem~\ref{eq:opt_sym}. This follows since the loss function $\ell(\U,\U)$ is invariant under rotations, hence a weight matrix $\U$ cannot be optimal if there exists a rotation matrix $\Rr$ such that $R(\U\Rr,\U\Rr) < R(\U,\U)$. Now, while the objective in Problem~\ref{opt:lower_sym} is a lower bound on the  objective in Problem~\ref{eq:opt_sym}, by Theorem~\ref{thm:equalization}, we know that any weight matrix can be equalized. Thus, it follows that the minimum of the two problems coincide. Although Problem~\ref{opt:lower_sym} is still non-convex, it is easier to study owing to a simpler form of the regularizer.
Figure~\ref{fig:autoenc} shows how optimization landscape changes with different dropout rates for a single hidden layer linear autoencoder with one dimensional input and output and with a hidden layer of width two. 
\end{remark}

%% file: asymmetric.tex
\section{Single hidden-layer linear  networks}\label{sec:asym}

Next, we consider the more general setting of a shallow linear network with a single hidden layer. Recall,  that the goal is to find weight matrices $\U, \V$ that solve
\vspace*{-6pt}
\begin{equation}\label{eq:opt_asym}
\min_{\U\in \R^{d_1 \times r},\V\in \R^{d_2 \times r}} \ell(\U,\V) + \lambda \sum_{i=1}^r \| \u_i \|^2 \| \v_i \|^2. 
\end{equation}
As in the previous section, we note that the loss function is rotation invariant, i.e. $\ell(\U\Rr,\V\Rr)=\ell(\U,\V)$ for any rotation matrix $\Rr$, however the regularizer is not invariant to rotations. Furthermore, it is easy to verify that both the loss function and
the regularizer are invariant under rescaling of the incoming
and outgoing weights to hidden neurons.

\begin{remark}[Rescaling invariance]
The objective function in Problem~\eqref{eq:main1} is invariant under rescaling of weight matrices, i.e. invariant to transformations of the form $\bar\U=\U\D$, $\bar\V=\V\D^{-1}$, where $\D$ is a diagonal matrix with positive entries. This follows since $\bar\U\bar\V^\top=\U\D\D^{-\top}\V^\top=\U\V^\top$, so that $\ell(\bar\U,\bar\V)=\ell(\U,\V)$, and also $R(\bar\U,\bar\V)=R(\U,\V)$ since $$\sum_{i=1}^{r}{\| \bar{\u}_i \|^2\| \bar{\v}_i \|^2}=\sum_{i=1}^{r}{\| d_i{\u_i} \|^2\| \frac1{d_i} \v_i \|^2}=\sum_{i=1}^{r}{\| \u_i \|^2\| \v_i \|^2}.$$
\end{remark}

As a result of rescaling invariance, $f(\bar\U,\bar\V)=f(\U,\V)$. Now, following similar arguments as in the previous section, we define $\n_{\u,\v}=(\|\u_1\| \|\v_1\|,\ldots,\|\u_r\| \|\v_r\|)$, and note that
\begin{align*}
R(\U,\V)&=\lambda\sum_{i=1}^{r}{\| \u_i \|^2 \| \v_i \|^2} =\frac{\lambda}{r} \| \1_r \|^2 \| \n_{\u,\v} \|^2 \\
&\geq \frac{\lambda}{r} \langle \1_r,\n_{\u,\v} \rangle^2 = \frac{\lambda}{r} \left( \sum_{i=1}^{r}{\| \u_i \| \| \v_i \|}\right)^2,
\end{align*}
where the inequality is due to Cauchy-Schwartz, and the lower bound is achieved if and only if $\n_{\u,\v}$ is a scalar multiple of $\1_r$, i.e. \emph{iff}  $\|\u_i\|\|\v_i\|=\|\u_1\|\|\v_1\|$ for all $i = 1,\ldots,r$. 
This observation motivates the following definition.

\begin{definition}[Jointly equalized weight matrices, equalized linear networks]
A pair of weight matrices $(\U,\V)\in \R^{d_1 \times r}\times \R^{d_2 \times r}$ is said to be \textit{jointly equalized} if
$\|\u_i\| \|\v_i\| = \|\u_1\| \|\v_1\|$ for all $i \in [r]$. A single hidden-layer linear network is said to be equalized if the product of the norms of the incoming and outgoing weights are equal for all hidden nodes. Equivalently, a single hidden-layer network  parametrized by weight matrices $\U, \V$, is equalized if $\U, \V$ are jointly equalized. An orthogonal transformation $\Rr\in \R^{r\times r}$ is an \textit{equalizer} of a single hidden-layer network $h_{\U,\V}$ parametrized by weight matrices $\U,\V$, if $h_{\U\Rr,\V\Rr}$ is equalized. The network $h_{\U,\V}$ (the pair$(\U,\V)$) then are said to be \textit{jointly equalizable} by $\Rr$.
\end{definition}

Note that Theorem~\ref{thm:equalization} only guarantees the existence of an equalizer for an autoencoder with tied weights. It does not inform us regarding the existence of a rotation matrix that jointly equalizes a general network parameterized by a pair of weight matrices $(\U,\V)$; in fact, it is not true in general that any pair $(\U,\V)$ is jointly equalizable. Indeed, the general case requires a more careful treatment. It turns out that while a given pair of matrices $(\U,\V)$ may not be jointly equalizable there exists a pair $(\tilde{\U},\tilde{\V})$ that is jointly equalizable and implements the same network function, i.e. 
$h_{\tilde{\U}, \tilde{\V}} = h_{\U, \V}$. Formally, we state the following result.
\begin{theorem}\label{thm:asym_equalization}
For any given pair of weight matrices $(\U,\V) \in \R^{d_1\times r}\times \R^{d_2\times r}$,  there exists another pair $(\tilde\U,\tilde\V) \in \R^{d_1\times r} \times \R^{d_2\times r}$ and a rotation matrix $\Rr \in \R^{r \times r}$ such that
$h_{\tilde\U, \tilde\V} = h_{\U, \V}$ and $h_{\tilde\U, \tilde\V}$ is jointly equalizable by $\Rr$. 
Furthermore, for $\bar\U:=\tilde\U\Rr$ and $\bar\V:=\tilde\V\Rr$ it holds that 
$\| \bar\u_i \|^2 = \| \bar\v_i \|^2 = \frac1r \| \U\V^\top \|_*$ for $i=1,\ldots,r$. 
\end{theorem}

Theorem~\ref{thm:asym_equalization} implies that for any network $h_{\U, \V}$ there exists an equalized network $h_{\bar\U, \bar\V}$ such that $h_{\bar\U, \bar\V} = h_{\U, \V}$. Hence, it is always possible to reduce the objective by equalizing the network, and a network $h_{\U,\V}$ is globally optimal only if it is equalized. 

\begin{theorem}\label{thm:equalized_asym}
If $(\U,\V)$ is a global optimum of Problem~\ref{eq:opt_asym}, then $\U,\V$ are jointly equalized. Furthermore, it holds that $$R(\U,\V)=\frac{\lambda}{r}\left(\sum_{i=1}^{r}{\|\u_i\|\|\v_i\|}\right)^2 = \frac{\lambda}{r} \| \U\V^\top \|_*^2$$
\end{theorem}

\begin{remark} 
As in the case of autoencoders with tied weights in Section~\ref{sec:sym}, a complete characterization of the implicit bias of dropout is given by considering the global optimality in terms of the network, i.e. in terms of the product of the weight matrices $\U \V^\top$. Not surprisingly, even in the case of single hidden-layer networks,    dropout promotes sparsity, i.e. favors low-rank weight matrices. 
\end{remark} 

\begin{theorem}\label{thm:asym_global}
For any $j\in[r]$, let  $\kappa_j:=\frac{1}{j}\sum_{i=1}^{j}{\lambda_i(\M)}$. Furthermore, define $\rho:=\max\{ j \in [r]: \ \lambda_j(\M) > \frac{\lambda j \kappa_j}{r+\lambda j} \}$. Then, if $(\U_*,\V_*)$ is a global optimum of Problem~\ref{eq:opt_asym}, it satisfies that 
$\U_*\V_*^\top = \cS_{\frac{\lambda \rho \kappa_\rho}{r+\lambda \rho}}(\M)$.
\end{theorem}

%% file: convergence.tex
\section{Geometry of the Optimization Problem}\label{sec:landscape}

While the focus in Section~\ref{sec:sym} and Section~\ref{sec:asym} was on understanding the implicit bias of dropout in terms of the global optima of the resulting regularized learning problem, here we focus on computational aspects of dropout as an optimization procedure. Since dropout is a first-order method (see Algorithm~\ref{alg:dropout}) and the landscape of  Problem~\ref{eq:opt_sym} is highly non-convex, we can perhaps only hope to find a \textit{local} minimum, that too provided if the problem has no degenerate saddle points~\cite{lee2016gradient,ge2015escaping}. Therefore, in this section, we pose the following questions: \emph{What is the implicit bias of dropout in terms of local minima? Do local minima share anything with global minima structurally or in terms of the objective? Can dropout find a local optimum?}

For the sake of simplicity of  analysis, we focus on the case of autoencoders with tied weight as in Section~\ref{sec:sym}. We show in Section~\ref{sec:local} that (a) local minima of Problem~\ref{eq:opt_sym} inherit the same implicit bias as the global optima, i.e. all local minima are equalized. Then, in Section~\ref{sec:ss}, we show that for sufficiently small regularization parameter,  %$\lambda$, 
(b) there are no spurious local minima, i.e. all local minima are global, and (c) all saddle points are non-degenerate (see Definition~\ref{def:ss}).

\subsection{Implicit bias in local optima}\label{sec:local}

We begin by recalling that the loss $\ell(\U, \U)$ is rotation invariant, i.e. $\ell(\U\Rr, \U\Rr)=\ell(\U, \U)$ for any rotation matrix $\Rr$. Now, if the weight matrix $\U$ were not equalized, then there exist indices $i,j\in[r]$ such that $\|\u_i \| > \| \u_j \|$.  We show that it is easy to design a rotation matrix (equal to identity everywhere expect for columns $i$ and $j$) that moves mass from $\u_i$ to $\u_j$ such that the 
difference in the norms of the corresponding columns of $\U\Rr$ decreases strictly while leaving the norms of other columns invariant. In other words, this rotation strictly reduces the regularizer and hence the objective. Formally, this implies the following result. 
\begin{lemma}\label{lem:local}
All local optima of Problem~\ref{eq:opt_sym} are equalized, i.e. if $\U$ is a local optimum, then $\|\u_i\|=\|\u_j\|$ $\forall i, j \in [r].$
\end{lemma}

Lemma~\ref{lem:local} unveils a fundamental property of dropout. As soon as we perform dropout in the hidden layer -- \emph{no matter how small the dropout rate} -- all local minima become equalized. 

\subsection{Landscape properties}\label{sec:ss}

Next, we characterize the solutions to which dropout (i.e. Algorithm~\ref{alg:dropout}) converges. We do so by understanding the optimization landscape of Problem~\ref{eq:opt_sym}. Central to our analysis, is the following notion of \textit{strict saddle property}.

\begin{definition}[Strict saddle point/property]\label{def:ss}
Let $f:\cU\to \R$ be a twice differentiable function and let $\U\in \cU$ be a critical point of  $f$. Then, $\U$ is a \textit{strict saddle point} of $f$~if the Hessian of $f$ at $\U$ has at least one  negative eigenvalue, i.e. $\lambda_{\min}(\nabla^2 f(\U)) < 0$. Furthermore, $f$ satisfies  \textit{strict saddle property} if all saddle points of $f$ are strict saddle.
\end{definition}
Strict saddle property ensures that for any critical point $\U$ that is not a local optimum, the Hessian has a significant negative eigenvalue which allows first order methods such as gradient descent (GD) and stochastic gradient descent (SGD) to escape saddle points and converge to a local minimum~\citep{lee2016gradient,ge2015escaping}. Following this idea, there has been a flurry of works on studying the landscape of different machine learning problems, including low rank matrix recovery~\citep{bhojanapalli2016global}, generalized phase retrieval problem~\citep{sun2016geometric}, matrix completion~\citep{ge2016matrix}, deep linear networks~\citep{kawaguchi2016deep}, matrix sensing and robust PCA~\citep{ge2017no} and tensor decomposition~\citep{ge2015escaping}, making a case for global optimality of first order methods.

For the special case of no regularization (i.e. $\lambda=0$; equivalently, no dropout), Problem~\ref{eq:opt_sym} reduces to standard squared loss minimization which has been shown to have no spurious local minima and satisfy strict saddle property (see, e.g. \citep{baldi1989neural, jin2017escape}). However, the  regularizer induced by dropout can potentially introduce new spurious local minima as well as degenerate saddle points. Our next result establishes that that is not the case, at least when the dropout rate is sufficiently small. 

\begin{theorem}\label{thm:main_geometry}
For regularization parameter $\lambda<\frac{r\lambda_r(\M)}{\sum_{i=1}^{r}\lambda_i(\M)-r\lambda_r(\M)}$, 
(a) all local minima of Problem~\ref{eq:opt_sym} are global, and (b) all saddle points 
are strict saddle points.
\end{theorem}

A couple of remarks are in order. First, Theorem~\ref{thm:main_geometry} guarantees that any critical point $\U$ that is not a global optimum is a strict saddle point, i.e. $\nabla^2 f(\U,\U)$ has a negative eigenvalue. This property allows first order methods, such as dropout given in Algorithm~\ref{alg:dropout}, to escape such saddle points.
Second, note that the guarantees in Theorem~\ref{thm:main_geometry} hold when the regularization parameter $\lambda$ is sufficiently small. Assumptions of this kind are common in the literature (see, for example \citep{ge2017no}). While this is a {\textit{sufficient}} condition for the result in Theorem~\ref{thm:main_geometry}, it is not clear if it is {\textit{necessary}}.

%% file: factorization.tex
\begin{figure*}[t]
\centering
\begin{tabular}{ccc}
$\lambda = 0.1$  & $\lambda = 0.5$  & $\lambda = 1$  \\ 
\hspace*{-22pt} 
\includegraphics[width=0.36\textwidth]{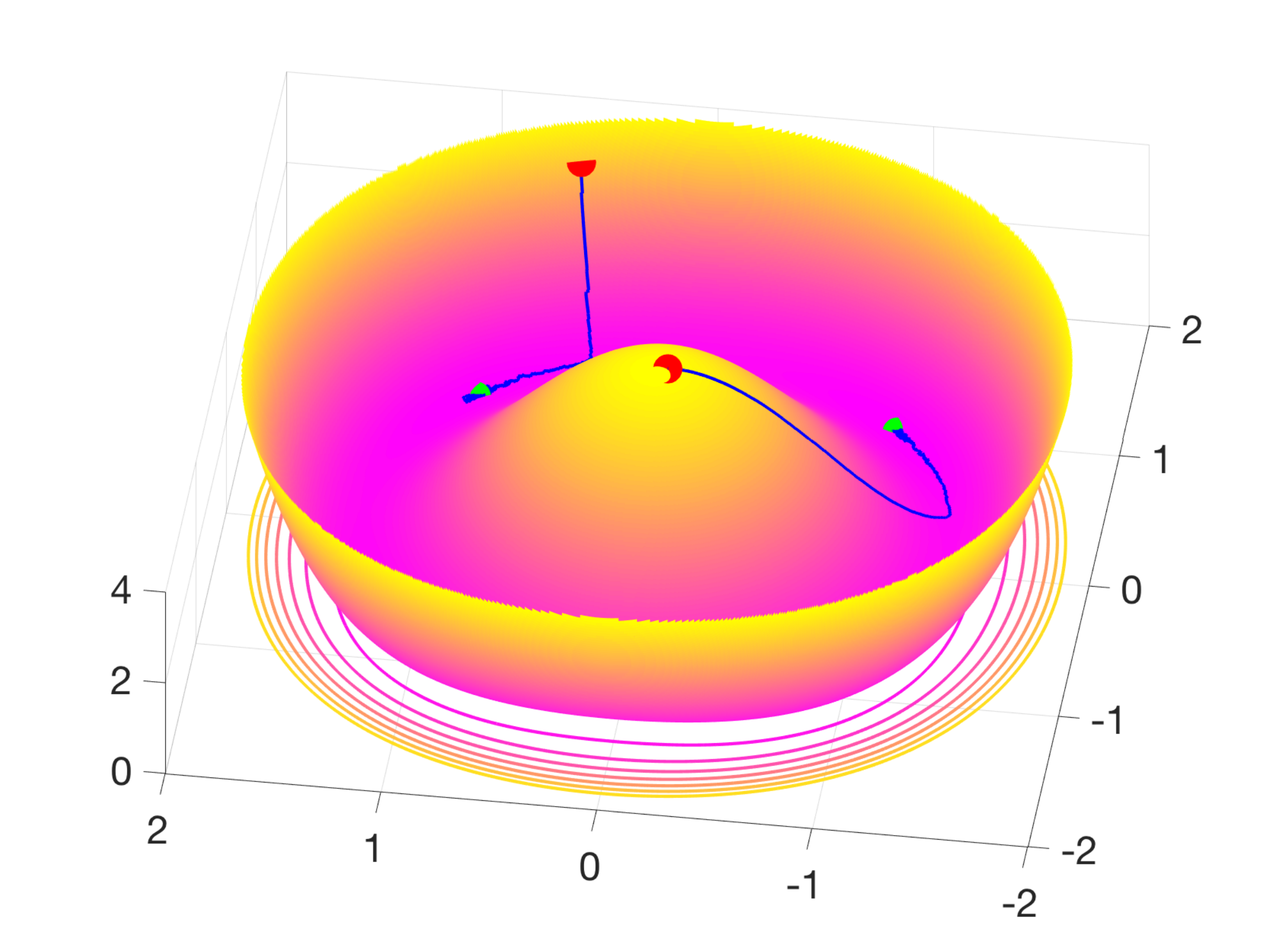}
&
\hspace*{-34pt} 
\includegraphics[width=0.36\textwidth]{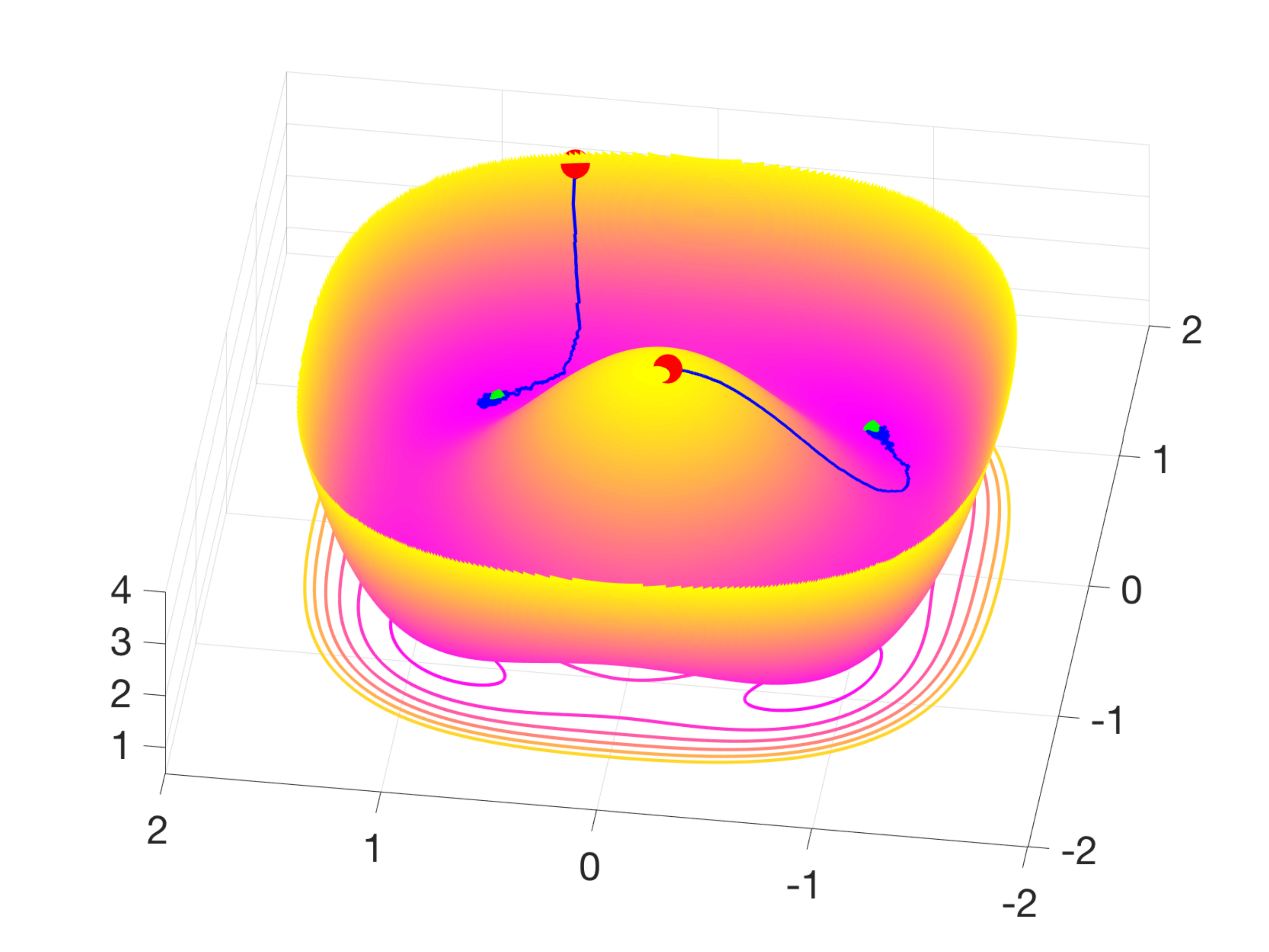}
&
\hspace*{-34pt} 
\includegraphics[width=0.36\textwidth]{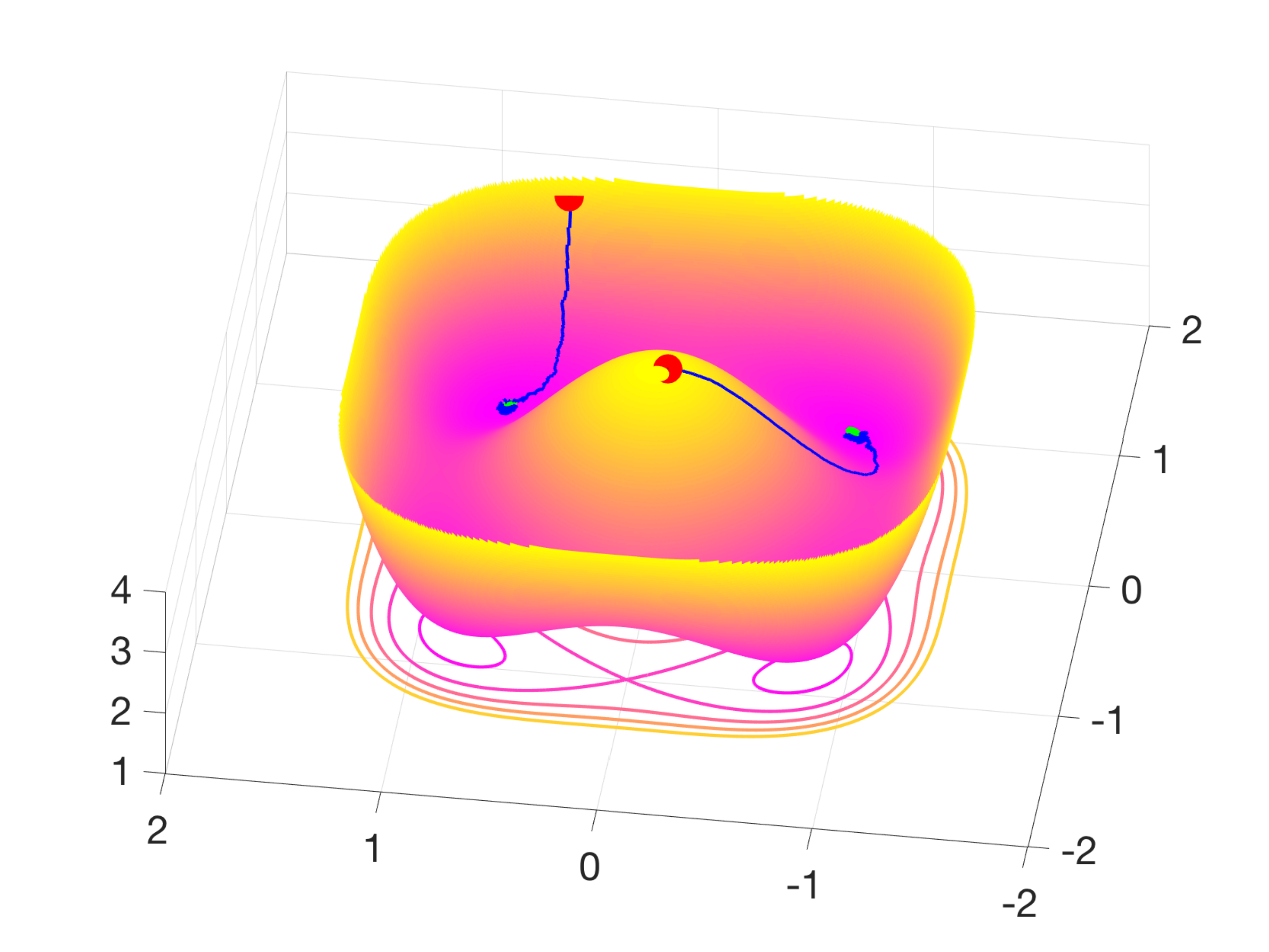}
\end{tabular}
\caption{\label{fig:conv_1d}{Convergence of dropout (Algorithm~\ref{alg:dropout}) from two different initialization (marked in red circles) to a global optimum of Problem~\ref{eq:opt_sym} (marked in green circles), for the simple case of scalar $\M$ (one dimensional input and output) and $r=2$. It can be seen that dropout quickly converges to a global optimum, which is equalized (i.e. weights are parallel to $(\pm1,\pm1)$) regardless of the value of the regularization parameter, $\lambda = 0.1$ (left), $\lambda =  0.5$ (middle) and $\lambda = 1.0$ (right).}}
\end{figure*}

\section{Matrix Factorization with Dropout}
\label{sec:factorization}
\begin{algorithm}[b]
\caption{\label{alg:polytime}Polynomial time solver for Problem~\ref{eq:matrix_dropout_sym}}
\begin{algorithmic}[1]
\INPUT Matrix $\M\in \R^{d_2\times d_1}$ to be factorized, size of factorization $r$, regularization parameter $\lambda$
\STATE $\rho\gets\max\{ j \in [r]: \ \lambda_j(\M) > \frac{\lambda j \kappa_j}{r+\lambda j} \}$, \\ where $\kappa_j=\frac{1}{j}\sum_{i=1}^{j}{\lambda_i(\M)}$ for $j\in[r]$.
\STATE $\bar\M \gets \cS_{\frac{\lambda \rho \kappa_\rho}{r+\lambda \rho}}(\M)$
\STATE $(\U,\Sigma,\V)\gets \texttt{svd}(\bar\M)$
\STATE $\tilde\U \gets \U\Sigma^{\frac12}$, $\tilde\V \gets \V\Sigma^{\frac12}$
\STATE $\Rr\gets \texttt{EQZ}(\tilde\U)$ \COMMENT{Algorithm~\ref{alg:equalizer}}
\STATE $\bar\U\gets \tilde\U\Rr$, $\bar\V\gets \tilde\V\Rr$ 
\OUTPUT $\bar\U,\bar\V$ \COMMENT{global optimum of Problem~\ref{eq:matrix_dropout_sym}}
\end{algorithmic}
\end{algorithm}

The optimization problem associated with learning a shallow network, i.e. Problem~\ref{eq:opt_asym}, is closely related to the optimization problem for matrix factorization. Recall that in matrix factorization,  given a matrix $\M \in \R^{d_1 \times d_2}$, one seeks to find factors $\U, \V$ that minimize $\ell(\U,\V)=\|\M-\U\V^\top\|_F^2$.
Matrix factorization has recently been studied with dropout by~\citet{zhai2015dropout,he2016dropout} and \citet{Cavazza2017analysis} where at each iteration of gradient descent on the loss function, the columns of factors $\U, \V$ are dropped independently and with equal probability. Following~\citet{Cavazza2017analysis}, we can write the resulting problem as 
\begin{align}\label{eq:matrix_dropout_sym}
\minim{\U\in\R^{d_1\times r},\V\in\R^{d_2\times r}}{ \| \M - \U\V^\top \|_F^2 + \lambda \sum_{i=1}^{r}{\| \u_i \|^2 \| \v_i \|^2}}, 
\end{align}
which is identical to Problem~\ref{eq:opt_asym}. However, there are two key distinctions. First, we are interested in stochastic optimization problem whereas the matrix factorization problem is typically posed for a given matrix. Second, for the learning problem that we consider here, it is unreasonable to assume access to the true model (i.e. matrix $\M$). Nonetheless, many of the insights we develop here as well as the technical results and algorithmic contributions apply to matrix factorization. Therefore, the goal in this section is to bring to bear the results in Sections~\ref{sec:sym}, \ref{sec:asym} and \ref{sec:landscape} to matrix factorization. 

We note that Theorem~\ref{thm:asym_global} and Theorem~\ref{thm:asym_equalization}, both of which hold for matrix factorization, imply that there is a polynomial time algorithm to solve the matrix factorization problem. In order to find a global optimum of Problem~\ref{eq:matrix_dropout_sym},  we first compute the optimal $\bar\M = \tilde\U \tilde\V^\top$ using  shrinkage-thresholding operation (see Theorem~\ref{thm:asym_global}). A global optimum $(\bar\U, \bar\V)$ is then obtained by joint equalization of $(\tilde\U,\tilde\V)$ (see Theorem~\ref{thm:asym_equalization})  using Algorithm~\ref{alg:equalizer}. The whole procedure is described in Algorithm~\ref{alg:polytime}. Few remarks are in order. 
 
\begin{remark}[Computational cost of Algorithm~\ref{alg:polytime}]
It is easy to check that computing $\rho,\bar\M,\tilde\U$ and $\tilde\V$ requires computing a rank-$r$ SVD of $\M$, which costs $O(d^2r)$, where $d=\max\{ d_1,d_2 \}$. Algorithm~\ref{alg:equalizer} entails computing $\G_\U = \U^\top \U$, which costs $O(r^2d)$ and the cost of each iterate of Algorithm~\ref{alg:equalizer} is dominated by computing the eigendecomposition which is $O(r^3)$. Overall, the computational cost of Algorithm~\ref{alg:polytime} is $O(d^2 r+ dr^2+r^4)$.
\end{remark}

\begin{remark}[Universal Equalizer]\label{remark:universal}
While Algorithm~\ref{alg:equalizer} is efficient (only linear in the dimension) for any rank $r$, there is a more effective equalization procedure when $r$ is a power of $2$. In this case, we can give a universal equalizer which works simultaneously for all matrices in $\R^{d\times r}$. Let $\U\in \R^{d\times r}$, $r=2^k$, $k \in \N$ and let $\U=\W \Sigma \V^\top$ be its full SVD. The matrix $\tilde\U = \U\Rr$ is equalized, where $\Rr=\V\Zz_k$ and
\begin{equation*}
\Zz_k := \begin{cases} 
      1 & k = 1 \\
      2^{\frac{-k+1}{2}}\begin{bmatrix} &\Zz_{k-1} &\Zz_{k-1} \\ &-\Zz_{k-1} &\Zz_{k-1} \end{bmatrix} & k>1 
   \end{cases}.
   \end{equation*}
\end{remark}

Finally, we note that Problem~\ref{eq:matrix_dropout_sym} is an instance of regularized matrix factorization which has recently received considerable attention in the machine learning  literature~\citep{ge2016matrix,ge2017no,haeffele2017structured}. These works show that the saddle points of a class of regularized matrix factorization problems have certain ``nice'' properties (i.e. escape directions characterized by negative curvature around saddle points) which allow variants of first-order methods such as perturbed gradient descent~\citep{ge2015escaping,jin2017escape} to converge to a local optimum. Distinct from that line of research, we completely characterize the set of 
global optima of Problem~\ref{eq:matrix_dropout_sym}, and provide a polynomial time algorithm to find a global optimum.

The work most similar to the matrix factorization problem we consider in this section is that of  \citet{Cavazza2017analysis}, with respect to which we make several important contributions: (I) \citet{Cavazza2017analysis} characterize optimal solutions only in terms of the product of the factors, and not in terms of the factors themselves, whereas we provide globally optimal solutions in terms of the factors; (II) \citet{Cavazza2017analysis} require the rank $r$ of the desired factorization to be variable and above some threshold, whereas we consider fixed rank-$r$ factorization for any $r$; (III)  \citet{Cavazza2017analysis} can only find low rank solutions using an adaptive dropout rate, which is not how dropout is used in practice, whereas we consider any fixed dropout rate; and (IV) we give an efficient poly time algorithm to find optimal factors. 

%% file: experiments.tex
\begin{figure*}[t]
\centering
\begin{tabular}{cccc}
$\lambda = 1$  & $\lambda = 0.5$  & $\lambda=0.1$  &  equalization  \\ %
\hspace*{-12pt} 
\includegraphics[width=0.27\textwidth]{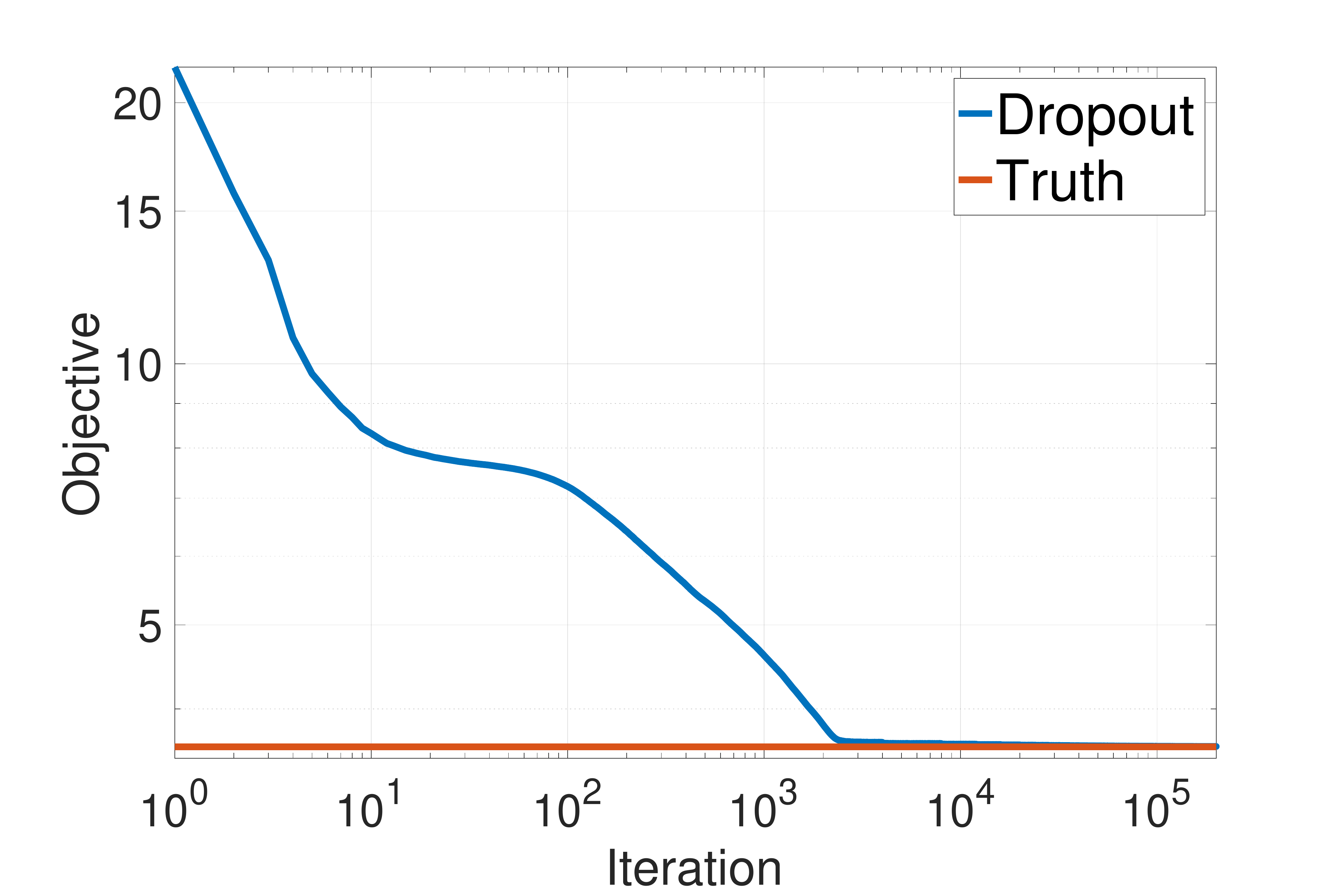}
&
\hspace*{-25pt} 
\includegraphics[width=0.27\textwidth]{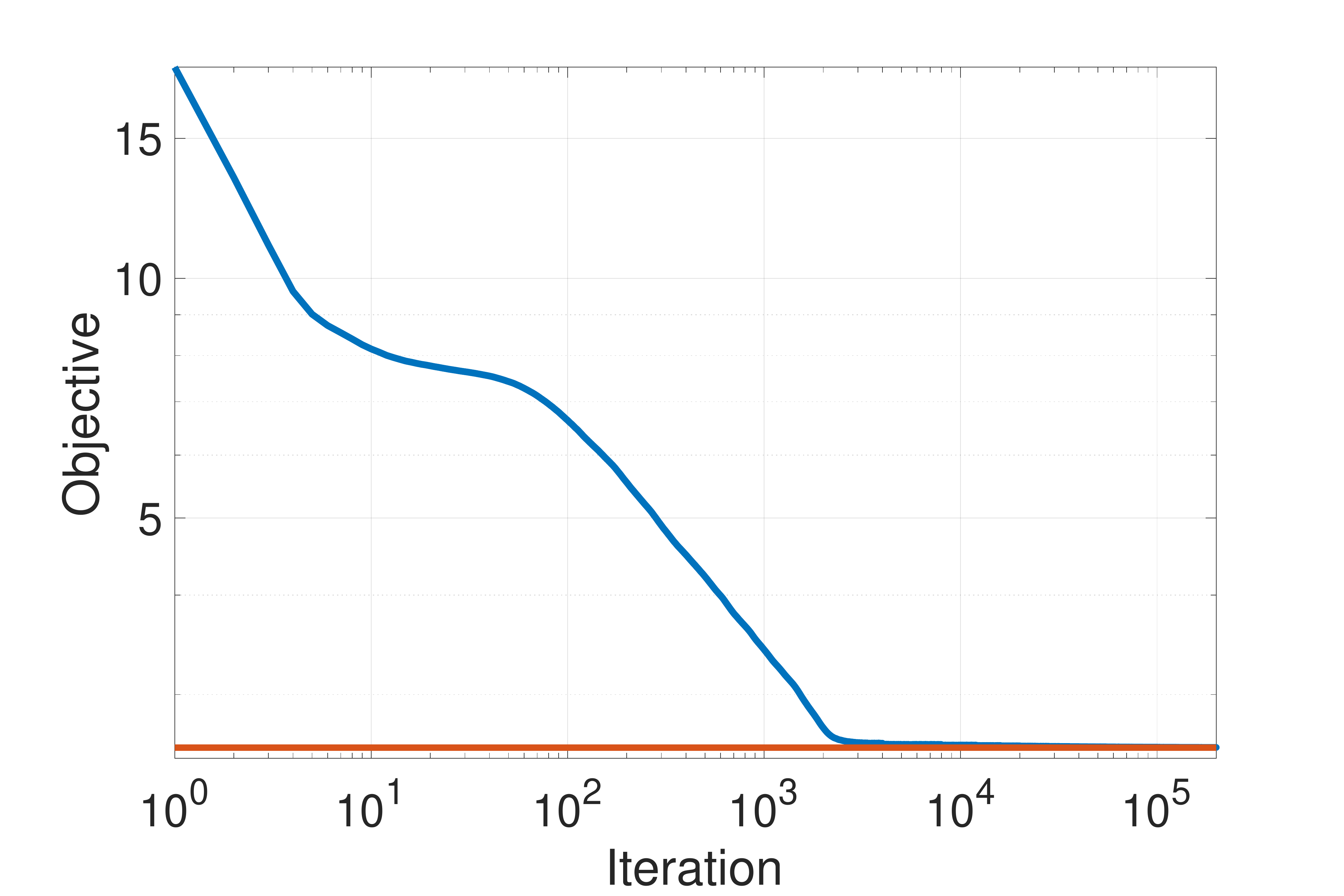}
&
\hspace*{-25pt} 
\includegraphics[width=0.27\textwidth]{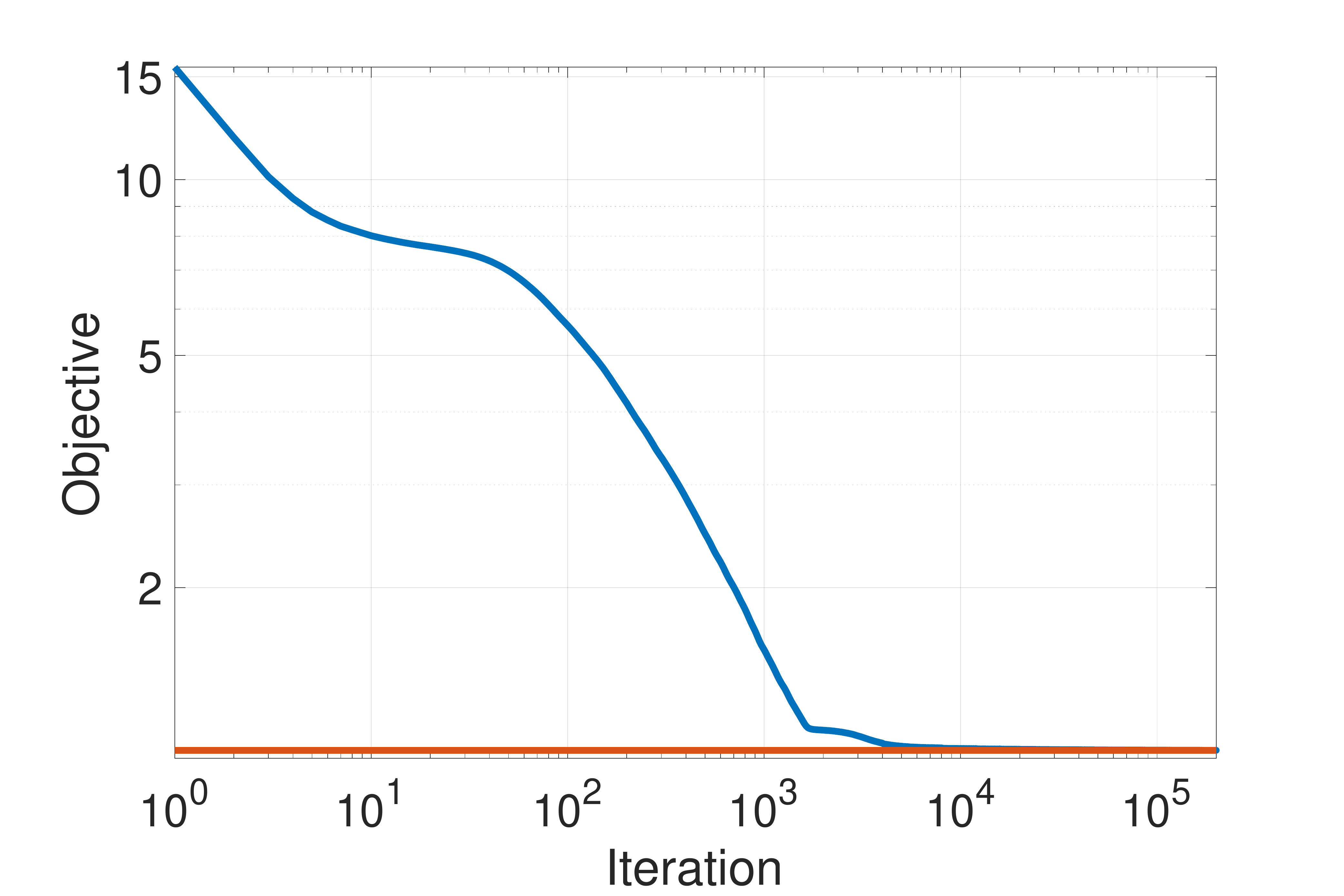}
&
\hspace*{-25pt} 
\includegraphics[width=0.27\textwidth]{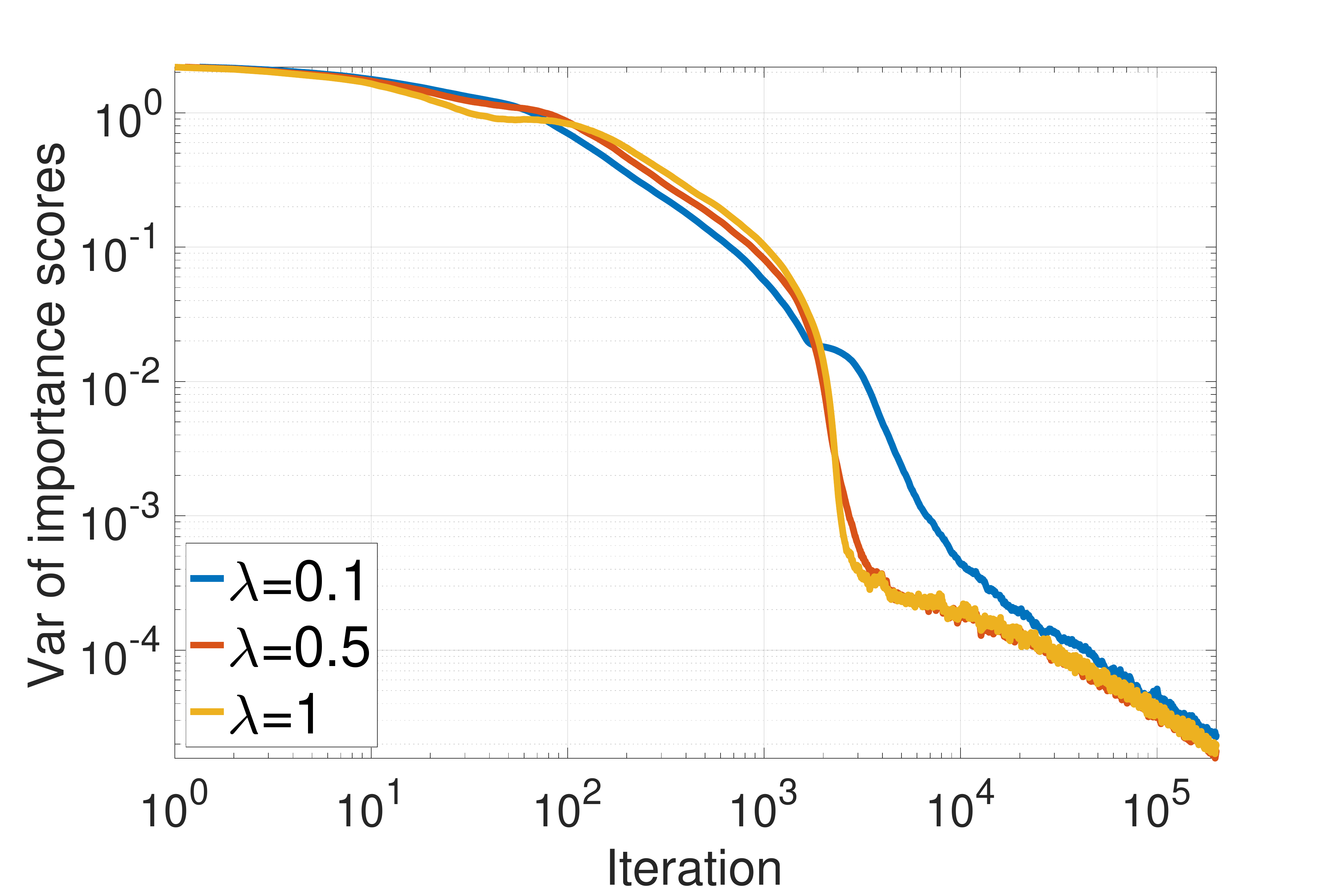}
\\
\hspace*{-12pt} 
\includegraphics[width=0.27\textwidth]{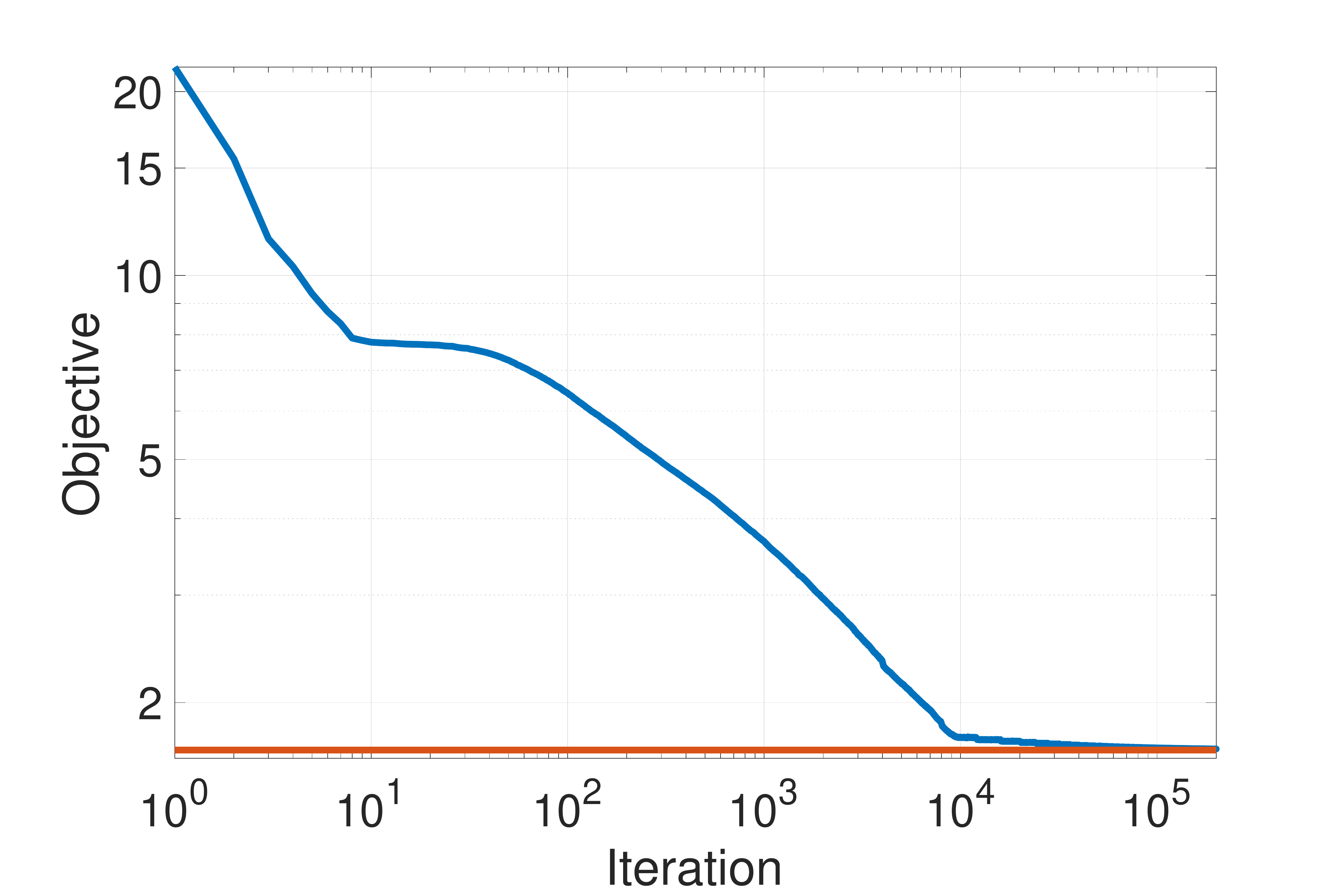}
&
\hspace*{-25pt} 
\includegraphics[width=0.26\textwidth]{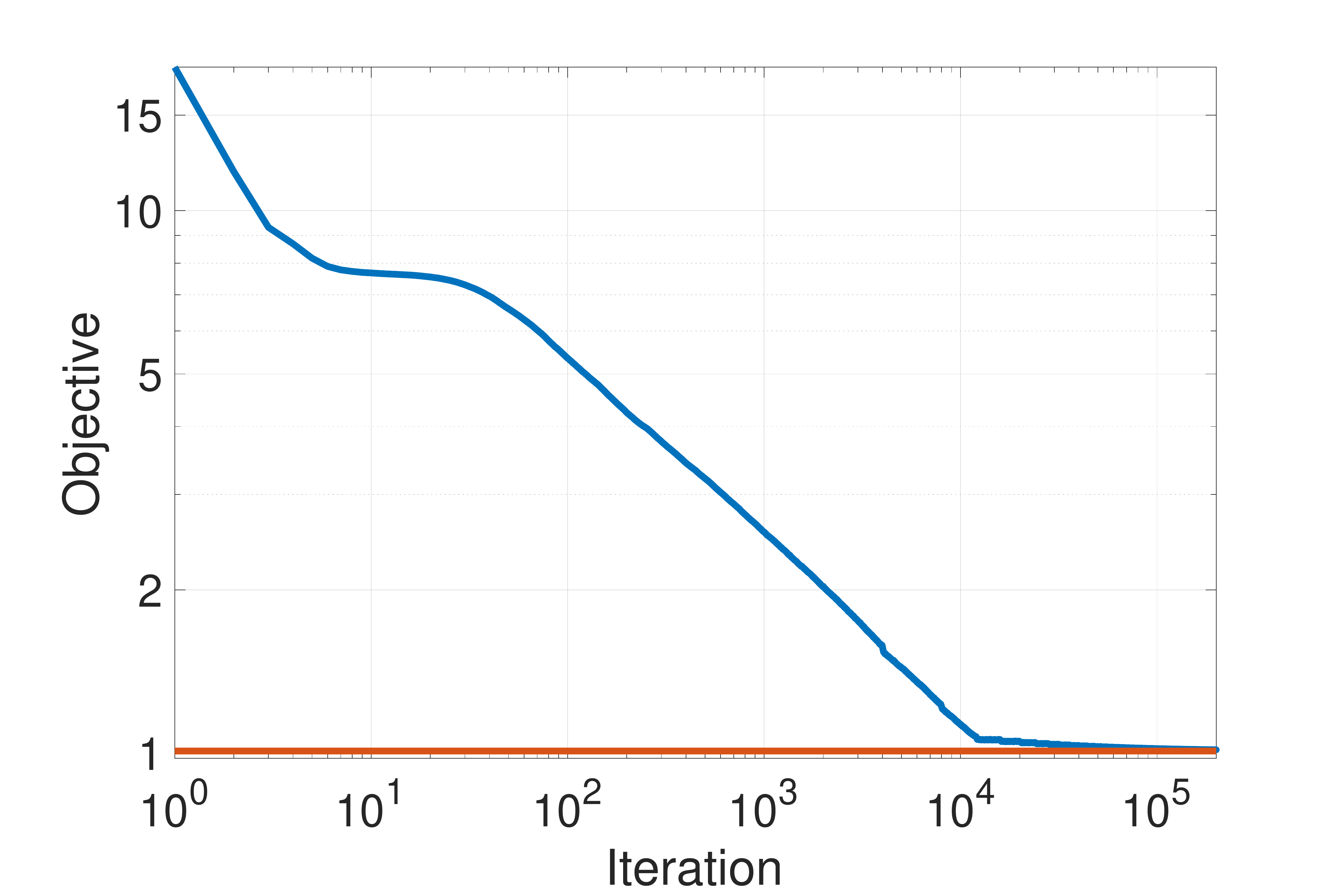}
&
\hspace*{-25pt} 
\includegraphics[width=0.26\textwidth]{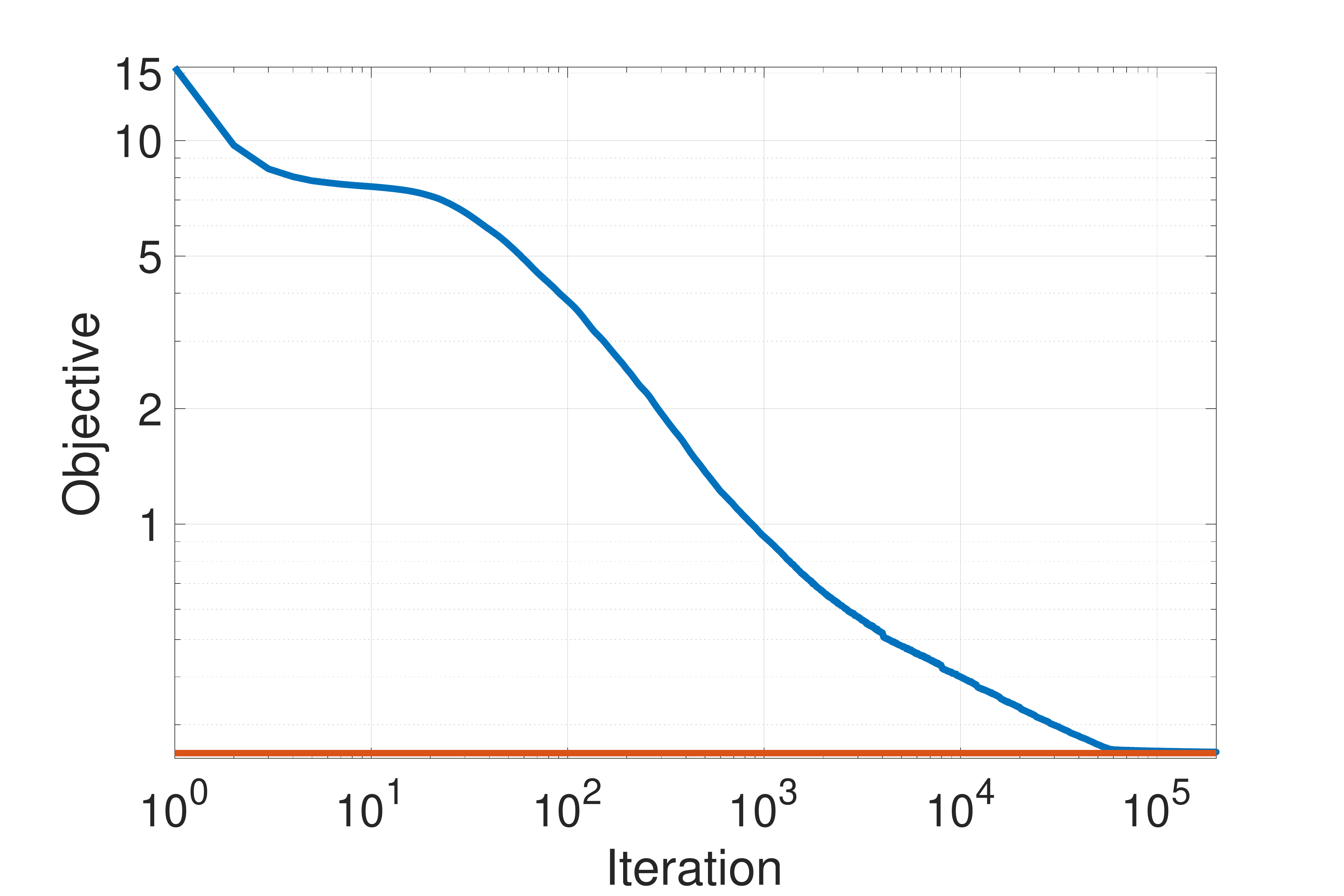}
&
\hspace*{-25pt} 
\includegraphics[width=0.26\textwidth]{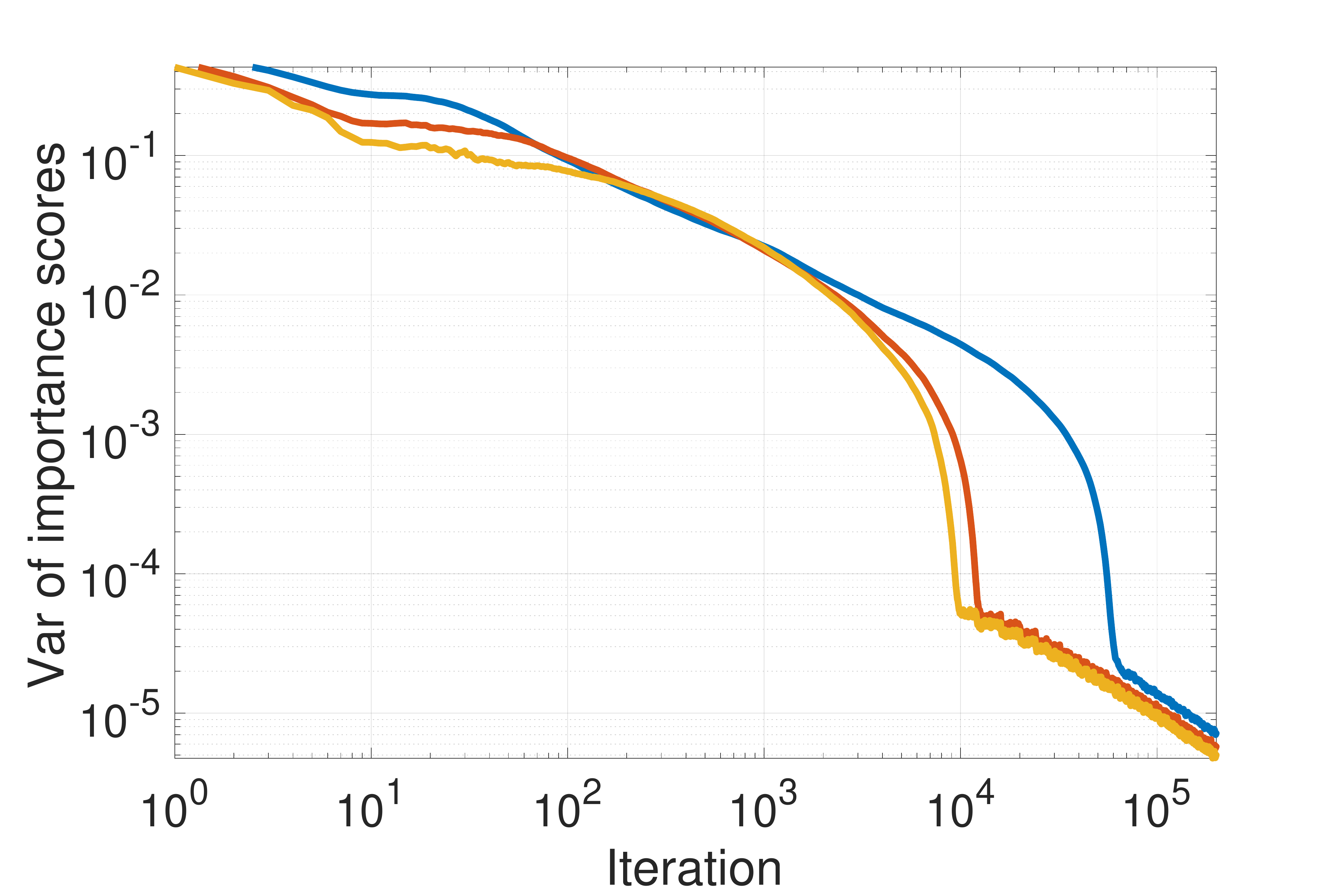}
\end{tabular}
\vspace*{-3pt}
\caption{\label{fig:conv_syn}{Dropout converges to global optima for different values of $\lambda \in \{ 0.1,0.5,1 \}$ and different widths of the hidden layer $r=20$ (top) and $r=80$ (bottom). The right column shows the variance of the product of column-wise norms for each of the weight matrices. As can be seen, the weight matrices become equalized very quickly since variance goes to zero.}}
\end{figure*}

\section{Empirical Results}\label{sec:experiments}

Dropout is a popular algorithmic technique used for avoiding overfitting when training large deep neural networks. The goal of this section is not to attest to the already well-established success of dropout. Instead, the purpose of this section is to simply confirm the theoretical results we showed in the previous section, as a proof of concept. 

We begin with a toy example in order to visually illustrate the optimization landscape. We use dropout to learn a simple linear auto-encoder with one-dimensional input and output (i.e. a network represented by a scalar $\M=2$) and a single hidden layer of width $r=2$. The input features are sampled for a standard normal distribution. Figure~\ref{fig:conv_1d} shows the optimization landscape along with the contours of the level sets, and a trace of iterates of dropout (Algorithm~\ref{alg:dropout}). The initial iterates and global optima (given by Theorem~\ref{thm:sym_global}) are shown by red and  green dots, respectively. Since at any global optimum the weights are equalized, the 
optimal weight vector in this case  is parallel to the vector $(\pm 1, \pm 1)$. We see that dropout converges to a global minimum. 

For a second illustrative experiment, we use  Algorithm~\ref{alg:dropout} to train a shallow linear network, where the input $\x\in\R^{80}$ is distributed according to the standard Normal distribution. The output $\y\in \R^{120}$ is generated as $\y=\M\x$, where  $\M\in\R^{120 \times 80}$ is drawn randomly by uniformly sampling the right and left singular subspaces and with a spectrum decaying  exponentially. Figure~\ref{fig:conv_syn} illustrates the behavior of Algorithm~\ref{alg:dropout} for different values of the regularization parameter ($\lambda \in \{ 0.1,0.5,1 \}$), and for different sizes of factors ($r\in \{ 20,80\}$). The curve in blue shows the objective value for the iterates of dropout, and the line in red shows the optimal value of the objective (i.e. objective for a global optimum found using  Theorem~\ref{thm:asym_global}). All plots are averaged over 50 runs of Algorithm~\ref{alg:dropout} (averaged over different random initializations, random realizations of Bernoulli dropout, as well as random draws of training examples).

To verify that the solution found by dropout actually has equalized factors, we consider the following measure. At each iteration, we compute the ``importance scores'', $\alpha_t^{(i)}=\|{\u_t}_i\| \|{\v_t}_i\|, \ i \in [r]$, where ${\u_t}_i$ and ${\v_t}_i$ are the $i$-th columns of $\U_t$ and $\V_t$, respectively. The rightmost panel of Figure~\ref{fig:conv_syn} shows the variance of $\alpha_t^{(i)}$'s,  over the hidden nodes $i \in [r]$, at each iterate $t$. Note that a high variance in $\alpha_t$ corresponds to large variation in the values of $\|{\u_t}_i\| \|{\v_t}_i\|$. When the variance is equal to zero, all importance scores are equal, thus the factors are equalized. We see that iterations of Algorithm~\ref{alg:dropout} decrease this measure monotonically, and the larger the value of $\lambda$, the faster the weights become equalized.

%% file: discussion.tex
\section{Discussion}
\vspace*{4pt}

There has been much effort in recent years to understand the theoretical underpinnings of dropout~(see~\citet{ baldi2013understanding, gal2016dropout, wager2013dropout, helmbold2015inductive}). In this paper, we study the implicit bias of dropout in shallow linear networks. We show that dropout prefers solutions with minimal path regularization which yield strong capacity control guarantees in deep learning. Despite being a non-convex optimization problem, we are able to fully characterize the global optima of the dropout objective. Our analysis shows that dropout favors low-rank weight matrices that are equalized. This theoretical finding confirms that dropout as a procedure uniformly allocates weights to different subnetworks, which is akin to preventing co-adaptation.
\vspace*{5pt}

We characterize the optimization landscape of learning autoencoders
with dropout. We first show that the local optima inherit the same implicit bias as global optimal, i.e. all local optima are equalized. Then, we show that for  sufficiently small dropout rates, there are no spurious local minima in the landscape, and all saddle points are non-degenerate. These properties suggest that dropout -- as an optimization procedure -- can efficiently converge to a globally optimal solution specified by our theorems.
\vspace*{5pt}

Understanding dropout in shallow linear networks is a prerequisite for understanding dropout in deep learning. We see natural extensions of our results in two directions: 1) shallow networks with non-linear activation function such as rectified linear units (ReLU) which have been shown to enable faster training~\cite{glorot2011deep} and are better understood in terms of the family of functions represented by ReLU-nets~\cite{arora2018understanding}, and 2) exploring the global optimality in deeper networks, even for linear activations.

%% file: appendix.tex
\clearpage
\appendix
\section{Auxiliary Lemmas}
In this section, we prove Lemma~\ref{lem:equiv} and a few auxiliary lemmas that we will need for the proofs of Theorem~\ref{thm:sym_global} and Theorem~\ref{thm:asym_global}.

\begin{lemma}\label{lem:equiv}
Let $\x \in \R^{d_2}$ be distributed according to distribution $\cD$ with $\E_{\x}[\x\x^\top]=\I$. Then, for $\ell(\U,\V):=\E_{\x}[\|\y-\U\V^\top\x\|^2]$ and 
%\begin{equation}
$f(\U,\V):=\E_{b,\x}[\| \y - \frac1{\theta} \U \diag(\b) \V^\top\x \|^2],$
%\end{equation}
it holds that 
\begin{equation}
f(\U,\V)=\ell(\U,\V)+\lambda\sum_{i=1}^{r}{\|\u_i\|^2 \|\v_i\|^2}.
\end{equation}
Furthermore, $\ell(\U,\V)=\|\M-\U\V^\top\|_F^2$.
\end{lemma}
\begin{proof}[Proof of Lemma~\ref{lem:equiv}]
The proof closely follows~\cite{Cavazza2017analysis}. Recall that $\y=\M\x$, for some unknown $\M\in \R^{d_2\times d_1}$. Observe that
\begin{align}\label{eq:equiv1}
f(\U,\V)&=\E_\x[\|\y\|^2]+\frac1{\theta^2}\E_{\b,\x}[\|\U\diag(\b)\V^\top\x \|^2] \nonumber\\
&- \frac{2}{\theta}\E_\x[\langle \M\x,\E_\b[\U\diag(\b)\V^\top]\x \rangle]
\end{align}
where we used the fact that $\y=\M\x$. We have the following set of equalities for the second term on the right hand side of Equation~\eqref{eq:equiv1}:
\begin{align}\label{eq:equiv2}
&\E_{\b,\x}[\|\U\diag(\b)\V^\top\x \|^2]=\E_\x \sum_{i=1}^{d_2}{\E_\b\left(\sum_{j=1}^{r}{u_{ij}b_j\v_j^\top \x}\right)^2} \nonumber\\
&=\E_\x \sum_{i=1}^{d_2}{\E_\b[\sum_{j,k=1}^{r}{u_{ij}u_{ik}b_j b_k (\v_j^\top \x) (\v_k^\top\x)}]}\nonumber\\
&=\E_\x \sum_{i=1}^{d_2}{\sum_{j,k=1}^{r}{u_{ij}u_{ik}(\theta^2 \1_{j\neq k} + \theta \1_{j=k}) (\v_j^\top \x) (\v_k^\top\x)}}\nonumber\\
&=\theta^2\E_\x[\|\U\V^\top\x \|^2] + (\theta-\theta^2)\E_\x \sum_{i=1}^{d_2}{\sum_{j=1}^{r}{u_{ij}^2 (\v_j^\top \x)^2}}\nonumber\\
&=\theta^2\E_\x[\|\U\V^\top\x \|^2] + (\theta-\theta^2) \sum_{j=1}^{r}\|\v_j\|^2{\sum_{i=1}^{d_2}{u_{ij}^2}}\nonumber\\
&=\theta^2\E_\x[\|\U\V^\top\x \|^2] + (\theta-\theta^2) \sum_{j=1}^{r}\|\v_j\|^2\|\u_j\|^2,
\end{align}
where the second to last equality follows because $\E_\x[(\v_j^\top\x)^2]=\v_j^\top\E_\x[\x\x^\top]\v_j=\|\v_j\|^2$. For the third term in Equation~\eqref{eq:equiv1} we have:
\begin{equation}\label{eq:equiv3}
\langle \M\x,\E_\b[\U\diag(\b)\V^\top]\x\rangle = \theta \langle \M\x,\U\V^\top\x \rangle
\end{equation}
Plugging Equations~\eqref{eq:equiv2}~and~\eqref{eq:equiv3} into~\eqref{eq:equiv1}, we get 
\begin{align}\label{eq:equiv4}
f(\U,\V)&=\E_\x[\|\y\|^2]+\E_\x[\|\U\V^\top\x\|^2]- 2\E_\x\langle \M\x,\U\V^\top\x \rangle\nonumber\\
&+\frac{1-\theta}{\theta}\sum_{i=1}^{r}{\|\u_i\|^2\|\v_i\|^2}
\end{align}
It is easy to check that the first three terms in Equation~\eqref{eq:equiv4} sum to $\ell(\U,\V)$. Furthermore, since for any $\A\in \R^{d_2\times d_1}$ it holds that $\|\A\x\|^2=\|\A\|_F^2$, we should have $\ell(\U,\V)=\| \M-\U\V^\top\|_F^2$.
\end{proof}

\begin{lemma}\label{lem:inverse}
For any pair of integers $\rho$ and $r$, and for any $\lambda \in \R_+$, it holds that $$(\I_\rho + \frac{\lambda}{r}\1\1^\top)^{-1}=\I_\rho - \frac{\lambda}{r+\lambda\rho}\1\1^\top.$$
\end{lemma}
Lemma~\ref{lem:inverse} is an instance of the Woodbury's  matrix identity. Here, we include a proof for completeness. 
\begin{proof}[Proof of Lemma~\ref{lem:inverse}]
The proof simply follows from the following set of equations.
\begin{align*}
&(\I_\rho + \frac{\lambda}{r}\1\1^\top)(\I_\rho - \frac{\lambda}{r+\lambda\rho}\1\1^\top)\\
&=\I_\rho + \frac{\lambda}{r}\1\1^\top  - \frac{\lambda}{r+\lambda\rho}\1\1^\top  - \frac{\lambda^2}{r(r+\lambda\rho)}\1\1^\top \1\1^\top\\
&=\I_\rho + \left( \frac{\lambda}{r}  - \frac{\lambda}{r+\lambda\rho}  - \frac{\rho \lambda^2}{r(r+\lambda\rho)} \right) \1\1^\top=\I_\rho
\end{align*}
\end{proof}

\begin{lemma}\label{lem:monotone}
Let $\lambda>0$ be a constant. Let $\a\in\R^d_+$ such that $a_i\geq a_{i+1}$ for all $i\in[d-1]$. For $\r\leq d$, let the function $g:[r] \to \R$ be defined as 
\begin{align*}
g(\rho) &:= \sum_{i=1}^{\rho}{\left( \frac{\lambda\sum_{k=1}^{\rho}{a_k}}{r+\lambda \rho} \right)^2} +\sum_{i=\rho+1}^{d}{a_i^2} \\
&+ \frac{\lambda}{r}\left( \sum_{i=1}^{\rho}{\left(a_i - \frac{\lambda\sum_{k=1}^{\rho}{a_k}}{r+\lambda \rho}\right)} \right)^2.
\end{align*}
Then $g(\rho)$ is monotonically non-increasing in $\rho$.
\end{lemma}

\begin{proof}[Proof of Lemma~\ref{lem:monotone}]
Let denote the sum of the top $\tau$ elements of $\a$ by $\h_{\tau}=\sum_{i=1}^{\tau}{a_i}$. Furthermore, let the sum of squared of $\tau$ bottom elements of $\a$ be denoted by $t_\tau=\sum_{i=\tau+1}^{d}{a_i^2}$. We can simplify $g(\rho)$ and give it in terms of $h_\rho$ and $t_\rho$ as follows:
\begin{align*}
g(\rho) =& \rho{\left( \frac{\lambda h_\rho}{r+\lambda \rho} \right)^2} +t_\rho + \frac{\lambda}{r}\left( \left( 1-\frac{\lambda\rho}{r+\lambda \rho} \right)h_\rho \right)^2 \\
=& \frac{\rho \lambda^2 + \lambda r}{(r+\lambda \rho)^2} \left( h_\rho \right)^2 + t_\rho \\
&= \frac{\lambda h_\rho^2}{r+\lambda \rho} + t_\rho
\end{align*}
It suffices to show that $g(\rho+1)\leq g(\rho)$ for all $\rho \in [r-1]$.
\begin{align*}
&g(\rho+1)= \frac{\lambda h_{\rho+1}^2}{r+\lambda \rho + \lambda} + t_{\rho+1} \\
=& \frac{\lambda}{r+\lambda \rho + \lambda} \left( h_\rho^2 + \lambda_{\rho+1}^2(\M) + 2\lambda_{\rho+1}(\M)h_\rho \right) \\
&-\lambda_{\rho+1}^2(\M) + t_\rho \\
=& g(\rho) - \frac{\lambda^2 h_\rho^2}{(r+\lambda\rho)(r+\lambda\rho + \lambda)} -\lambda_{\rho+1}^2(\M)\\
&+ \frac{\lambda}{r+\lambda \rho + \lambda} \left( \lambda_{\rho+1}^2(\M) + 2\lambda_{\rho+1}(\M)h_\rho  \right)  \\
=& g(\rho) - \frac{\lambda^2 h_\rho^2}{(r+\lambda\rho)(r+\lambda\rho + \lambda)}  - \frac{(r+\lambda\rho)\lambda_{\rho+1}^2(\M)}{r+\lambda \rho + \lambda} \\
&+ \frac{\lambda}{r+\lambda \rho + \lambda}\left( 2\lambda_{\rho+1}(\M)h_\rho  \right)\\
=& g(\rho) - \frac{\left( {\lambda h_\rho} - ({r+\lambda\rho})  \lambda_{\rho+1}^2(\M)  \right)^2}{{(r+\lambda\rho)(r+\lambda\rho + \lambda)}} \leq g(\rho).
\end{align*}
Hence $g(\rho)$ is monotonically non-increasing in $\rho$.
\end{proof}

\section{Proofs of Theorems in Section~\ref{sec:sym}}
\begin{proof}[Proof of Theorem~\ref{thm:equalization}]
Consider the matrix $\G_1:=\G_\U-\frac{\tr{\G_\U}}{r}\I_r$. We exhibit an orthogonal transformation $\Rr$, such that $\Rr^\top \G_1 \Rr$ is zero on its diagonal. Observe that $$\Rr^\top\G_\U\Rr = \Rr^\top\G_1\Rr +  \frac{\tr{\G_\U}}{r}\I_r,$$ so that all diagonal elements of $\G_\U$ are equal to $\frac{\tr{\G_\U}}r$, i.e. $\G_\U$ is equalized.

Our construction closely follows the proof of a classical theorem in matrix analysis, which states that any trace zero matrix is a commutator~\cite{albert1957matrices,kahan1999only}. For the zero trace matrix $\G_1$, we first show that there exists a unit vector $\w_{11}$ such that $\w_{11}^\top \G_1 \w_{11} = 0$.
\begin{claim}\label{claim:zero_rayleigh}
Assume $\G$ is a zero trace matrix and let $\G=\sum_{i=1}^{r}{\lambda_i\u_i\u_i^\top}$ be an eigendecomposition of $\G$. Then $\w=\frac{1}{\sqrt r}\sum_{i=1}^r{\u_i}$ has a vanishing Rayleigh quotient, that is, $\w^\top \G \w = 0$, and $\| \w \|=1$.
\end{claim}
\begin{proof}[Proof of Claim~\ref{claim:zero_rayleigh}]
First, we notice that $\w$ has unit norm $$\| \w \|^2 = \| \frac{1}{\sqrt r}\sum_{i=1}^r{\u_i} \|^2 = \frac1r \|\sum_{i=1}^r{\u_i} \|^2 = \frac1r \sum_{i=1}^r{\| \u_i \|^2} = 1.$$
It is easy to see that $\w$ has a zero Rayleigh quotient 
\begin{align*}
\w^\top \G \w &= (\frac{1}{\sqrt r}\sum_{i=1}^r{\u_i})^\top \G (\frac{1}{\sqrt r}\sum_{i=1}^r{\u_i}) \\
&= \frac1r \sum_{i,j=1}^r{\u_i\G\u_j} =  \frac1r \sum_{i=1}^r{\lambda_j \u_i^\top \u_j} =  \frac1r \sum_{i=1}^r{\lambda_i} =  0.
\end{align*}
\end{proof}
Let $\W_1:=[\w_{11}, \w_{12},\cdots, \w_{1d}]$ be such that $\W_1^\top \W_1 = \W_1\W_1^\top = \I_d$. Observe that $\W_1^\top\G_1\W_1$ has zero on its first diagonal elements $$\W_1^\top\G_1\W_1=\begin{bmatrix} &0 &\b_1^\top \\ &\b_1 &\G_2  \end{bmatrix}$$ The principal submatrix $\G_2$ also has a zero trace. With a similar argument, let $\w_{22}\in \R^{d-1}$ be such that $\|\w_{22}\|=1$ and $\w_{22}^\top\G_2\w_{22}=0$ and define $\W_{2} = \begin{bmatrix} &1 &0 &0 &\cdots &0 \\ &\0 &\w_{22} &\w_{23} &\cdots &\w_{2d}\end{bmatrix}\in \R^{d\times d}$ such that $\W_2^\top\W_2=\W_2\W_2^\top=\I_{d}$, and observe that $$(\W_1\W_2)^\top\G_1(\W_1\W_2)=\begin{bmatrix} &0 &\cdot & \cdots \\ &\cdot &0  &\cdots \\ &\vdots &\vdots &\G_2  \end{bmatrix}.$$  This procedure can be applied recursively so that for the {\textit{equalizer}} $\Rr=\W_1\W_2\cdots\W_d$ we have $$\Rr^\top\G_1\Rr=\begin{bmatrix} &0 &\cdot & \cdots &\cdot \\ &\cdot &0  &\cdots  &\cdot \\ &\vdots &\vdots &\ddots &\vdots \\ &\cdot &\cdot &\cdot &0  \end{bmatrix}.$$
\end{proof}

\begin{proof}[Proof of Theorem~\ref{thm:equalized_sym}]

Let us denote the squared column norms of $\U$ by $\n_\u=(\| \u_1 \|^2,\ldots,\| \u_r\|^2)$. Observe that for any weight matrix $\U$:
\begin{align*}
R(\U,\U)&=\lambda\sum_{i=1}^{r}{\| \u_i \|^4} =\frac{\lambda}{r} \| \1_r \|^2 \| \n_\u \|^2 \\
&\geq \frac{\lambda}{r} \langle \1_r,\n_\u \rangle^2 = \frac{\lambda}{r} \left( \sum_{i=1}^{r}{\| \u_i \|^2}\right)^2 =\frac{\lambda}{r} \| \U \|_F^4, 
\end{align*}
where $\1_r \in \R^r$ is the vector of all ones and the inequality is due to Cauchy-Schwartz. Hence, the regularizer is lower bounded by $\frac{\lambda}{r}\| \U \|_F^4$, with equality if and only if $\n_\u$ is parallel to $\1_r$, i.e. when $\U$ is equalized. Now, if $\U$ is not equalized, by Theorem~\ref{thm:equalization} there exist a rotation matrix $\Rr$ such that $\U\Rr$ is equalized, which implies $R(\U\Rr,\U\Rr)<R(\U,\U)$. Together with rotational invariance of the loss function, this gives a contradiction with global optimality $\U$. Hence, if $\U$ is a global optimum then it is equalized and we have $R(\U,\U) = \lambda \sum_{i=1}^{r}{\| \u_i \|^4} = \frac{\lambda}{r}\| \U\|_F^4$.
\end{proof}

\begin{proof}[Proof of Theorem~\ref{thm:sym_global}]
By Theorem~\ref{thm:equalized_sym}, if $\W$ is an optimum of Problem~\ref{eq:opt_sym}, then it holds that $\lambda\sum_{i=1}^{r}{\|\w_i\|^4} = \frac{\lambda}{r}\| \W \|_F^4$. Also, by Theorem~\ref{thm:equalization}, it is always possible to equalize any given weight matrix. Hence, Problem~\ref{eq:opt_sym} reduces to the following problem:
\begin{equation}\label{eq:sym_equalized}
\minim{\W\in \R^{d\times r}}{\| \M - \W\W^\top \|_F^2 + \frac{\lambda}{r}\|\W\|_F^4}
\end{equation}
Let $\M=\U_\M \Lambda_\M \U_\M^\top$ and $\W =\U_\W \Sigma_\W \V_\W^\top$ be an eigendecomposition of $\M$ and a full SVD of $\W$ respectively, such that $\lambda_i(\M)\geq \lambda_{i+1}(\M)$ and $\sigma_{i}(\W) \geq \sigma_{i+1}(\W)$ for all $i\in[d-1]$. Rewriting objective of Problem~\ref{eq:sym_equalized} in terms of these decompositions gives: 
\begin{align*}
&{\| \M - \W\W^\top \|_F^2 + \frac{\lambda}{r}\|\W\|_F^4} \\
&={{\| \U_\M \Lambda_\M \U_\M^\top -  \U_\W \Sigma_\W \Sigma_\W^\top \U_\W^\top \|_F^2 + \frac{\lambda}{r}\|\U_\W \Sigma_\W \V_\W^\top\|_F^4}}\\
&={{\| \Lambda_\M  -  \U' \Sigma_\W \Sigma_\W^\top \U'^\top \|_F^2 + \frac{\lambda}{r}\|\Sigma_\W\|_F^4}}\\
&={{\| \Lambda_\M \|_F^2 + \| \Lambda_\W \|_F^2 - 2\langle \Lambda_\M, \U'\Lambda_\W \U'^\top \rangle + \frac{\lambda}{r}\left( \tr(\Lambda_\W)\right)^2}}
\end{align*}
where $\Lambda_\W := \Sigma_\W\Sigma_\W^\top$ and $\U'=\U_\M^\top \U_\W$. By Von Neumann's trace inequality, for a fixed $\Sigma_\W$ we have that $$\langle \Lambda_\M,\U' \Lambda_W \U'^\top \rangle \leq  \sum_{i=1}^{d}{\lambda_i(\M)\lambda_i(\W)},$$ where the equality is achieved when $\Lambda_i(\W)$ have the same ordering as $\Lambda_i(\M)$ and $\U' = \I$, i.e. $\U_\M=\U_\W$. Now, Problem~\ref{eq:sym_equalized} is reduced to 
\begin{align*}
&\minim{\substack{\|\Lambda_\W\|_0\leq r,\\ \Lambda_\W \geq 0}}{\| \Lambda_\M  -  \Lambda_\W \|_F^2 + \frac{\lambda}{r}\left(\tr(\Lambda_\W)\right)^2} \\
&= \minim{\bar{\lambda} \in \R^r_+}{\sum_{i=1}^{r}{\left(\lambda_i(\M) - \bar\lambda_i\right)^2} \!+\! \sum_{i=r+1}^{d}{\lambda_i^2(\M)} \!+\! \frac{\lambda}{r}\left( \sum_{i=1}^{r}{\bar\lambda_i} \right)^2}
\end{align*}
 The Lagrangian is given by
 \begin{align*}
 L(\bar\lambda,\alpha)&=\sum_{i=1}^{r}{\left(\lambda_i(\M) - \bar\lambda_i\right)^2} + \sum_{i=r+1}^{d}{\lambda_i^2(\M)}  \\
 &+ \frac{\lambda}{r}\left( \sum_{i=1}^{r}{\bar\lambda_i} \right)^2 - \sum_{i=1}^{r}{\alpha_i\bar\lambda_i}
 \end{align*}
  The KKT conditions ensures that at the optima it holds for all $i \in [r]$ that 
\begin{align*}
&\bar\lambda_i \geq 0 , \ \alpha_i \geq 0 , \ \bar\lambda_i\alpha_i = 0 \\
& 2(\bar\lambda_i-\lambda_i(\M)) + \frac{2\lambda}{r}\left(\sum_{i=1}^{r}{\bar\lambda_i}\right) - \alpha_i =0
\end{align*}
Let $\rho = |{i: \bar\lambda_i > 0}|\leq r$ be the number of nonzero $\bar\lambda_i$. For $i = 1,\ldots,\rho$ we have $\alpha_i = 0$, hence 
\begin{align*}
&\bar\lambda_i + \frac{\lambda}{r}\left(\sum_{i=1}^{\rho}{\bar\lambda_i}\right) = \lambda_i(\M) \\
&\implies (\I_\rho + \frac{\lambda}{r}\1\1^\top)\bar\lambda_{1:\rho} = \lambda_{1:\rho}(\M) \\
&\implies \bar\lambda_{1:\rho} = (\I_\rho - \frac{\lambda}{r+\lambda\rho}\1\1^\top)\lambda_{1:\rho}(\M) \\
&\implies \bar\lambda_{1:\rho} = \lambda_{1:\rho}(\M)-\frac{\lambda\rho\kappa_\rho}{r+\lambda\rho}\1_\rho \\
&\implies \Lambda_\W = (\Lambda_\M - \frac{\lambda\rho\kappa_\rho}{r+\lambda\rho}\I_d)_+
\end{align*}
where $ \kappa_\rho:=\frac1{\rho}\sum_{i=1}^{\rho}{\lambda_i(\M)}$ and the second implication is due to Lemma~\ref{lem:inverse}. It only remains to find the optimal $\rho$. Let's define the function 
\begin{align*}
g(\rho) &:= \sum_{i=1}^{\rho}{\left(\lambda_i(\M) - \bar\lambda_i\right)^2} + \!\!\!\sum_{i=\rho+1}^{d}{\lambda_i^2(\M)} + \frac{\lambda}{r}\!\left( \sum_{i=1}^{\rho}{\bar\lambda_i}\! \right)^2\\
&=\sum_{i=1}^{\rho}{\left( \frac{\lambda\sum_{k=1}^{\rho}{\lambda_k(\M)}}{r+\lambda \rho} \right)^2} +\sum_{i=\rho+1}^{d}{\lambda_i(\M)^2} \\
&+ \frac{\lambda}{r}\left( \sum_{i=1}^{\rho}{\left(\lambda_i(\M) - \frac{\lambda\sum_{k=1}^{\rho}{\lambda_k(\M)}}{r+\lambda \rho}\right)} \right)^2.
\end{align*}

By Lemma~\ref{lem:monotone}, $g(\rho)$ is monotonically non-increasing in $\rho$, hence $\rho$ should be the largest \textit{feasible} integer, i.e. $$\rho=\max\{ j: \ \lambda_j > \frac{\lambda j \kappa}{r+\lambda j} \}.$$
\end{proof}

\begin{proof}[Proof of Remark~\ref{remark:universal}]
For $\tilde\U$ to have equal column norms, it suffices to show that $\tilde\U^\top \tilde\U$ is constant on its diagonal. Next, we note that 
\begin{align*}
\tilde\U^\top \tilde\U&=\Rr^\top \U^\top \U \Rr\\
&=(\V\Zz_k)^\top  (\W \Sigma \V^\top)^\top (\W \Sigma \V^\top) (\V\Zz_k)\\
&=\Zz_k^\top \V^\top  \V \Sigma \W^\top \W \Sigma \V^\top \V\Zz_k \\
&= \Zz_k^\top \Sigma^2 \Zz_k
\end{align*}
It remains to show that for any diagonal matrix $\D$, $\Zz_k^\top \D \Zz_k$ is diagonalized. First note that $$\Zz_2 \Zz_2^\top = \frac{1}{2}\begin{bmatrix} &1 &1 \\ &-1 &1 \end{bmatrix}\begin{bmatrix} &1 &-1 \\ &1 &1 \end{bmatrix}=\I_2$$ so that $\Zz_2$ is indeed a rotation. By induction, it is easy to see that $\Zz_k$ is a rotation for all $k$. Now, we show that $\Zz_k$ equalizes any diagonal matrix $\D$. Observe that 
$$[\Zz_k^\top \D \Zz_k]_{ii} = \sum_{i=1}^{2^{k-1}}{D_{ii}z_{ji}^2} = \sum_{i=1}^{2^{k-1}}{D_{ii}2^{-k+1}} =2^{1-k}{\tr \D}$$
 so that all the diagonal elements are identically equal to the average of the diagonal elements of $\D$.
\end{proof}

\section{Proofs of Theorems in Section~\ref{sec:asym}}

\begin{proof}[Proof of Theorem~\ref{thm:asym_equalization}]
Let $\U\V^\top=\W\Sigma\Y^\top$ be a compact SVD of $\U\V^\top$. Define $\tilde\U:=\W\Sigma^{1/2}$ and $\tilde\V:=\Y\Sigma^{1/2}$ and observe that $\tilde\U\tilde\V^\top = \U\V^\top$. Furthermore, let $\G_{\bar\U}=\tilde\U^\top \tilde\U$ and $\G_{\tilde\V}=\tilde\V^\top \tilde\V$ be their Gram matrices. Observe that $\G_{\tilde\U}=\G_{\tilde\V}=\Sigma$. Hence, by Theorem~\ref{thm:equalization}, there exists a rotation $\Rr$ such that $\bar\V:=\tilde\V\Rr$ and $\bar\U:=\tilde\U\Rr$ are equalized, with $\|\bar\u_i \|^2 = \| \bar\v_i \|^2 = \frac1r \tr{\Sigma}$.
\end{proof}

\begin{proof}[Proof of Theorem~\ref{thm:equalized_asym}]
Define $$\n_{\u,\v}=(\|\u_1\| \|\v_1\|,\ldots,\|\u_r\| \|\v_r\|)$$ and observe that
\begin{align*}
R(\U,\V)&=\lambda\sum_{i=1}^{r}{\|\u_i\|^2\|\v_i\|^2} \\
&=\frac{\lambda}r \| \n_{\u,\v} \|^2 \| \1_r \|^2 \geq \frac{\lambda}r \langle \n_{\u,\v} , \1_r \rangle^2 \\
&= \frac{\lambda}{r}\left(\sum_{i=1}^{r}{\|\u_i\|\|\v_i\|}\right)^2
\end{align*}
where the inequality is due to Cauchy-Schwartz, and it holds with equality if and only if $\n_{\u,\v}$ is parallel to $\1_r$. Let $(\bar\U,\bar\V)$ be a global optima of Problem~\ref{eq:opt_asym}. The inequality above together with Theorem~\ref{thm:asym_equalization} imply that $\bar\U$ and $\bar\V$ should be jointly equalized up to dilation transformations, hence the first equality claimed by the theorem.

To see the second equality, note that if $\U$ and $\V$ are jointly equalized, then $$\|\u_i\|^2=\| \v_i \|^2=\frac{1}{r}\tr{\Sigma},$$ where $\Sigma$ is the matrix of singular values of $\U\V^\top$. Hence, 
\begin{align*}
R(\U,\V)= \frac{\lambda}{r}\left(\sum_{i=1}^{r}{\|\u_i\|\|\v_i\|}\right)^2 &= \frac{\lambda}{r}\left(\frac1r \sum_{i=1}^{r}{\tr{\Sigma}}\right)^2\\
&= \frac{\lambda}{r}\left(\tr{\Sigma} \right)^2
\end{align*}
 which is equal to $\frac{\lambda}{r}\| \bar\U\bar\V^\top \|_*^2$ as claimed.
\end{proof}

\begin{proof}[Proof of Theorem~\ref{thm:asym_global}]
By Theorem~\ref{thm:equalized_asym}, if $(\X,\Y)$ is an optimum of Problem~\ref{eq:opt_asym}, then it holds that $$\lambda\sum_{i=1}^{r}{\|\x_i\|^2\|\y_i\|^2} = \frac{\lambda}{r}\| \X\Y^\top \|_*^2.$$ Hence, Problem~\ref{eq:opt_asym} reduces to the following problem:
\begin{equation}\label{eq:sym_equalized2}
\minim{\X\in \R^{d_1\times r},\Y\in \R^{d_2\times r}}{\| \M - \X\Y^\top \|_F^2 + \frac{\lambda}{r}\|\X\Y^\top\|_*^2}
\end{equation}
Let $\M=\U_\M \Sigma_\M \V_\M^\top$ and $\W:=\X\Y^\top =\U_\W \Sigma_\W \V_\W^\top$ be full SVDs of $\M$ and $\W$ respectively, such that $\sigma_i(\M)\geq \sigma_{i+1}(\M)$ and $\sigma_{i}(\W) \geq \sigma_{i+1}(\W)$ for all $i\in[d-1]$ where $d=\min\{ d_1,d_2 \}$. Rewriting objective of Problem~\ref{eq:sym_equalized2} in terms of these decompositions, 
\begin{align*}
&{\| \M - \X\Y^\top \|_F^2 + \frac{\lambda}{r}\|\X\Y^\top\|_*^2} \\
&={{\| \U_\M \Sigma_\M \V_\M^\top -  \U_\W \Sigma_\W \V_\W^\top \|_F^2 + \frac{\lambda}{r}\|\U_\W \Sigma_\W \V_\W^\top\|_*^2}}\\
&={{\| \Sigma_\M  -  \U' \Sigma_\W \V'^\top \|_F^2 + \frac{\lambda}{r}\|\Sigma_\W\|_*^2}}\\
&={{\| \Sigma_\M \|_F^2 + \| \Sigma_\W \|_F^2 - 2\langle \Sigma_\M, \U'\Sigma_\W \U'^\top \rangle + \frac{\lambda}{r} \|\Sigma_\W\|_*^2}}
\end{align*}
where $\U'=\U_\M^\top \U_\W$. By Von Neumann's trace inequality, for a fixed $\Sigma_\W$ we have that $\langle \Sigma_\M,\U' \Sigma_\W \U'^\top \rangle \leq  \sum_{i=1}^{d}{\sigma_i(\M)\sigma_i(\W)}$, where the equality is achieved when $\Sigma_i(\W)$ have the same ordering as $\Sigma_i(\M)$ and $\U' = \I$, i.e. $\U_\M=\U_\W$. Now, Problem~\ref{eq:sym_equalized2} is reduced to 
\begin{align*}
&\minim{\substack{\|\Sigma_\W\|_0\leq r,\\ \Sigma_\W \geq 0}}{\| \Sigma_\M  -  \Sigma_\W \|_F^2 + \frac{\lambda}{r}\| \Sigma_\W\|_*^2} \\
&= \minim{\bar{\sigma} \in \R^r_+}{\sum_{i=1}^{r}{\left(\sigma_i(\M) \! - \bar\sigma_i\right)^2} + \!\!\! \sum_{i=r+1}^{d}{\sigma_i^2(\M)} + \frac{\lambda}{r}\left( \sum_{i=1}^{r}{\bar\sigma_i}\!\! \right)^2}
\end{align*}
 The Lagrangian is given by
 \begin{align*}
 L(\bar\lambda,\alpha)&=\sum_{i=1}^{r}{\left(\sigma_i(\M) - \bar\sigma_i\right)^2} + \sum_{i=r+1}^{d}{\sigma_i^2(\M)}  \\
 &+ \frac{\lambda}{r}\left( \sum_{i=1}^{r}{\bar\sigma_i} \right)^2 - \sum_{i=1}^{r}{\alpha_i\bar\sigma_i}
 \end{align*}
  The KKT conditions ensures that $\forall i=1,\ldots,r$, 
\begin{align*}
&\bar\sigma_i \geq 0 , \ \alpha_i \geq 0 , \ \bar\sigma_i\alpha_i = 0 \\
& 2(\bar\sigma_i-\sigma_i(\M)) + \frac{2\lambda}{r}\left(\sum_{i=1}^{r}{\bar\sigma_i}\right) - \alpha_i =0
\end{align*}
Let $\rho = |{i: \bar\sigma_i > 0}|\leq r$ be the number of nonzero $\bar\sigma_i$. For $i = 1,\ldots,\rho$ we have $\alpha_i = 0$, hence 
\begin{align*}
&\bar\sigma_i + \frac{\lambda}{r}\left(\sum_{i=1}^{\rho}{\bar\sigma_i}\right) = \sigma_i(\M) \\
&\implies (\I_\rho + \frac{\lambda}{r}\1\1^\top)\bar\sigma_{1:\rho} = \sigma_{1:\rho}(\M) \\
&\implies \bar\sigma_{1:\rho} = (\I_\rho - \frac{\lambda}{r+\lambda\rho}\1\1^\top)\sigma_{1:\rho}(\M) \\
&\implies \bar\sigma_{1:\rho} = \sigma_{1:\rho}(\M)-\frac{\lambda\rho\kappa_\rho}{r+\lambda\rho}\1_\rho \\
&\implies \Sigma_\W = (\Sigma_\M - \frac{\lambda\rho\kappa_\rho}{r+\lambda\rho}\I_d)_+
\end{align*}
where $ \kappa_\rho=\frac1{\rho}\sum_{i=1}^{\rho}{\sigma_i(\M)}$ and the second implication holds since $(\I_\rho + \frac{\lambda}{r}\1\1^\top)^{-1}=\I_\rho - \frac{\lambda}{r+\lambda\rho}\1\1^\top$. It only remains to find the optimal $\rho$.  
Let's define the function 
\begin{align*}
g(\rho) \!&:=\! \sum_{i=1}^{\rho}{\left(\sigma_i(\M) - \bar\sigma_i\right)^2} \!+\! \sum_{i=\rho+1}^{d}{\sigma_i^2(\M)} + \frac{\lambda}{r}\left( \sum_{i=1}^{\rho}{\bar\sigma_i} \right)^2\\
&=\sum_{i=1}^{\rho}{\left( \frac{\lambda\sum_{k=1}^{\rho}{\sigma_k(\M)}}{r+\lambda \rho} \right)^2} +\sum_{i=\rho+1}^{d}{\sigma_i(\M)^2} \\
&+ \frac{\lambda}{r}\left( \sum_{i=1}^{\rho}{\left(\sigma_i(\M) - \frac{\lambda\sum_{k=1}^{\rho}{\sigma_k(\M)}}{r+\lambda \rho}\right)} \right)^2.
\end{align*}

 By Lemma~\ref{lem:monotone}, $g(\rho)$ is monotonically non-increasing in $\rho$, hence $\rho$ should be the largest \textit{feasible} integer, i.e. $$\rho=\max\{ j: \ \sigma_j > \frac{\lambda j \kappa_j}{r+\lambda j} \}.$$
\end{proof}

\section{Proofs of Theorems in Sections~\ref{sec:landscape}}
In this section for ease of notation we let $\lambda_i$ denote $\lambda_i(\M)$. Furthermore, with slight abuse of notation we let $f(\U)$, $\ell(\U)$ and $R(\U)$ denote the objective, the loss function and the regularizer, respectively.

It is easy to see that the gradient of the objective of Problem~\ref{eq:opt_sym} is given by $$\nabla f(\U) = 4(\U\U^\top - \M)\U + 4\lambda\U \diag(\U^\top \U).$$
We first make the following important observation about the critical points of Problem~\ref{eq:opt_sym}. 

\begin{lemma}\label{lem:critical_preceq}
If $\U$ is a critical point of Problem~\ref{eq:opt_sym}, then it holds that $\U\U^\top \preceq \M$.
\end{lemma}
\begin{proof}[Proof of Lemma~\ref{lem:critical_preceq}]
Since $\nabla f(\U) = \0$, we have that $$(\M-\U\U^\top)\U = \lambda \U \diag(\U^\top \U)$$ multiply both sides from right by $\U^\top$ and rearrange to get 
\begin{equation}\label{eq:grad}
\M\U\U^\top = \U\U^\top\U\U^\top + \lambda\U\diag(\U^\top\U)\U^\top
\end{equation}
 Note that the right hand side is symmetric, which implies that the left hand side must be symmetric as well, i.e. $$\M\U\U^\top=(\M\U\U^\top)^\top =\U\U^\top\M,$$ so that $\M$ and $\U\U^\top$ commute. Note that in Equation~\eqref{eq:grad}, $\U\diag(\U^\top\U)\U^\top \succeq \0$. Thus, $\M\U\U^\top \succeq \U\U^\top\U\U^\top$.  Let $\U\U^\top=\W\Gamma\W^\top$ be a compact eigendecomposition of $\U\U^\top$. We get $$\M\U\U^\top = \M\W\Gamma\W^\top \succeq \U\U^\top\U\U^\top = \W\Gamma^2\W^\top.$$ Multiplying from right and left by $\W\Gamma^{-1}$ and $\W^\top$ respectively, we have that $$\W^\top \M \W \succeq \Gamma$$ which completes the proof.
\end{proof}
Lemma~\ref{lem:critical_preceq} allows us to bound different norms of the critical points, as will be seen later in the proofs.

To explore the landscape properties of Problem~\ref{eq:opt_sym}, we first focus on the non-equalized critical points in Lemma~\ref{lem:minima_eqz}. We show that the set of non-equalized critical points does not include any local optima. Furthermore, all such points are strict saddles. Therefore, we turn our focus to the equalized critical points in Lemma~\ref{lem:no_spurious}. We show all such points inherit the eigenspace of the input matrix $\M$. This allows us to give a closed-form characterization of all the equalized critical points in terms of the eigendecompostion of $\M$. We then show that if $\lambda$ is chosen appropriately, all such critical points that are not global optima, are strict saddle points.

\begin{lemma}\label{lem:minima_eqz}
All local minima of Problem~\ref{eq:opt_sym} are equalized. Moreover, all critical points that are not equalized, are strict saddle points.
\end{lemma}
\begin{proof}[Proof of Lemma~\ref{lem:minima_eqz}]
We show that if $\U$ is not equalized, then any $\epsilon$-neighborhood of $\U$ contains a point with objective strictly smaller than $f(\U)$. More formally, for any $\epsilon>0$, we exhibit a rotation $\Rr_\epsilon$ such that $\| \U-\U\Rr_\epsilon \|_F \leq \epsilon$ and $f(\U\Rr_\epsilon)<f(\U)$. Let $\U$ be a critical point of Problem~\ref{eq:opt_sym} that is not equalized, i.e. there exists two columns of $\U$ with different norms. Without loss of generality, let $\|\u_1\| > \|\u_2\|$. We design a rotation matrix $\Rr$ such that it is almost an isometry, but it moves mass from $\u_1$ to $\u_2$. Consequently, the new factor becomes ``less un-equalized''  and achieves a smaller regularizer, while preserving the value of the loss. To that end, define $$\Rr_\delta:=\begin{bmatrix} 
&\sqrt{1-\delta^2} &-\delta & \0 \\
 &\delta &\sqrt{1-\delta^2}  &\0 \\ 
 &\0 &\0 &\I_{r-2}
\end{bmatrix}$$
and let $\hat\U:=\U\Rr_\delta.$ It is easy to verify that $\Rr_\epsilon$ is indeed a rotation. First, we show that for any $\epsilon$, as long as $\delta^2\leq \frac{\epsilon^2}{2\tr(\M)}$, we have $\hat\U\in \cB_\epsilon(\U)$:
\begin{align*}
\| \U - \hat\U \|_F^2&=\sum_{i=1}^{r}\| \u_i - \hat\u_i\|^2\\
&=\| \u_1 - \sqrt{1-\delta^2}\u_1 - \delta\u_2 \ \|^2\\
&+\| \u_2 - \sqrt{1-\delta^2}\u_2 + \delta\u_1 \ \|^2\\
&=2(1-\sqrt{1-\delta^2})(\|\u_1\|^2+\|\u_2\|^2)\\
&\leq 2\delta^2\tr(\M) \leq \epsilon^2
\end{align*}
where the second to last inequality follows from Lemma~\ref{lem:critical_preceq}, because $\|\u_1\|^2+\|\u_2\|^2\leq \|\U\|_F^2 =\tr(\U\U^\top) \leq \tr(\M)$, and also the fact that $1-\sqrt{1-\delta^2}= \frac{1- 1+ \delta^2}{1+\sqrt{1-\delta^2}}\leq \delta^2$.

Next, we show that for small enough $\delta$, the value of the function at $\hat\U$ is strictly smaller than that of $\U$. Observe that 
\begin{align*}
\|\hat\u_1\|^2 &= (1-\delta^2)\|\u_1\|^2 + \delta^2\|\u_2\|^2+2\delta\sqrt{1-\delta^2}\u_1^\top\u_2\\
\|\hat\u_2\|^2 &= (1-\delta^2)\|\u_2\|^2 + \delta^2\|\u_1\|^2-2\delta\sqrt{1-\delta^2}\u_1^\top\u_2
\end{align*}
and the remaining columns will not change, i.e. for $i=3,\cdots,r$, $\hat\u_i=\u_i$. Together with the fact that $\Rr_\delta$ preserves the norms, i.e. $\|\U\|_F = \| \U\Rr_\delta\|_F$, we get 
\begin{equation}\label{eq:constant_sum}
\|\hat\u_1\|^2+\|\hat\u_2\|^2 = \|\u_1\|^2+\|\u_2\|^2.
\end{equation}
Let $\delta=-c\cdot \mathrm{sgn}(\u_1^\top\u_2)$ for a small enough $c>0$ such that $\|\u_2\| < \|\hat\u_2\| \leq \|\hat\u_1\| <\|\u_1\|$. Using Equation~\eqref{eq:constant_sum}, This implies that $\|\hat\u_1\|^4+\|\hat\u_2\|^4 < \|\u_1\|^4+\|\u_2\|^4$, which in turn gives us $R(\hat\U)<R(\U)$ and hence $f(\hat\U)<f(\U)$. Therefore, a non-equalized critical point cannot be local minimum, hence the first claim of the lemma.

We now prove the second part of the lemma. Let $\U$ be a critical point that is not equalized. To show that $\U$ is a strict saddle point, it suffices to show that the Hessian has a negative eigenvalue. In here, we exhibit a curve along which the second directional derivative is negative. Assume, without loss of generality that $\|\u_1\|>\|\u_2\|$ and consider the curve $$\Delta(t)\!:=\![(\sqrt{\!1\!-\!t^2\!}\!-\!1\!)\u_1 \!+\! t\u_2, (\sqrt{\!1\!-\!t^2\!}\!-\!1\!)\u_2 \!-\! t\u_1, \0_{d,r-2} ]$$ It is easy to check that for any $t\in \R$, $\ell(\U+\Delta(t))=\ell(\U)$ since $\U+\Delta(t)$ is essentially a rotation on $\U$ and $\ell$ is invariant under rotations. Observe that
\begin{align*}
&g(t):=f(\U+\Delta(t))\\
&=f(\U)+ \|\sqrt{1-t^2}\u_1 + t\u_2\|^4 - \|\u_1\|^4\\
&+ \|\sqrt{1-t^2}\u_2 - t\u_1\|^4 - \|\u_2\|^4\\
&=f(\U)-2t^2(\|\u_1\|^4+\|\u_2\|^4)+8t^2(\u_1\u_2)^2\\
&+\!4t^2 \|\u_1\|^2\|\u_2\|^2\!+\! 4t\sqrt{\!1\!-\!t^2}\u_1^\top\! \u_2(\|\u_1\|^2\!-\!\!\|\u_2\|^2)\!+\!O(t^3).
\end{align*}
The derivative of $g$ then is given as 
\begin{align*}
&g'(t)\!=\!-4t(\|\u_1\|^4 \!+\!\|\u_2\|^4)\!+\! 16t(\u_1\u_2)^2 \!+\! 8t\|\u_1\|^2\|\u_2\|^2\!\\
&+\! 4(\sqrt{\!1\!-\!t^2}-\frac{t^2}{\sqrt{\!1\!-\!t^2}})(\u_1^\top\! \u_2)(\|\u_1\|^2\!-\|\u_2\|^2)+O(t^2).
\end{align*}
Since $\U$ is a critical point and $f$ is continuously differentiable, it should hold that $g'(0)=4(\u_1^\top\u_2)(\|\u_1\|^2-\|\u_2\|^2)=0$. Since by assumption $\|\u_1\|^2-\|\u_2\|^2>0$, it should be the case that $\u_1^\top\u_2=0$. We now consider the second order directional derivative:
\begin{align*}
g''(0)&=-4(\|\u_1\|^4+\|\u_2\|^4)+16(\u_1\u_2)^2+8\|\u_1\|^2\|\u_2\|^2\\
&=-4(\|\u_1\|^2-\|\u_2\|^2)^2 < 0
\end{align*}
which completes the proof.
\end{proof}

We now focus on the critical points that are equalized, i.e. points $\U$ such that $\nabla f(\U)=\0$ and $\diag(\U^\top\U)=\frac{\|\U\|_F^2}{r}\I$. 

\begin{lemma}\label{lem:no_spurious}
Assume that $\lambda<\frac{r\lambda_r}{\sum_{i=1}^{r}\lambda_i-r\lambda_r}$.
Then all equalized local minima are global. All other equalized critical points are strict saddle points.
\end{lemma}
\begin{proof}[Proof of Lemma~\ref{lem:no_spurious}]
Let $\U=\W\Sigma\V^\top$ be a compact SVD of the rank-$r'$ weight matrix $\U$. We have: 
\begin{align*}
&\nabla f(\U) = 4(\U\U^\top - \M)\U + 4\lambda\U \diag(\U^\top \U)=\0\\
\implies & \U\U^\top\U + \frac{\lambda \| \U \|_F^2}{r} \U = \M\U\\
\implies & \W\Sigma^3\V^\top + \frac{\lambda \| \Sigma \|_F^2}{r} \W\Sigma\V^\top = \M\W\Sigma\V^\top\\
\implies & \Sigma^2 + \frac{\lambda \| \Sigma \|_F^2}{r} \I = \W^\top\M\W
\end{align*}
Since the left hand side of the above equality is diagonal, it implies that $\W\in\R^{d\times r'}$ corresponds to some $r'$ eigenvectors of $\M$. Let $\cE\subseteq [d], \ |\cE|=r'$ denote the set of eigenvectors of $\M$ that are present in $\W$. Note that the above is equivalent of the following system of linear equations: $$(\I+\frac{\lambda}{r}\1\1^\top)\sigma^2=\vec\lambda,$$ where $\sigma^2:=\diag(\Sigma^2)$ and $\vec\lambda=\diag(\W^\top\M\W)$. 
By Lemma~\ref{lem:inverse}, the solution to this linear system is given by 
\begin{equation}\label{eq:sol}
\sigma^2=(\I-\frac{\lambda}{r+\lambda r'})\vec\lambda.
\end{equation}
The set $\cE$ belongs to one of the following categories:
\begin{enumerate}
\item $\cE=[r'], \ r'=\rho$
\item $\cE=[r'], \ r'<\rho$
\item $\cE\neq[r']$
\end{enumerate}
The case $\cE=[r'], \ r'>\rho$ is excluded from the above partition, since whenever $\cE=[r']$, it should hold that $r'\leq\rho$. To see this, note that due to $\U=\W\Sigma\V^\top$ being a compact SVD of $\M$, it holds that $\sigma_j>0$ for all $j \in [r']$. Specifically for $j=r'$, plugging $\sigma_{r'}>0$ back to Equation~\eqref{eq:sol}, we get $$\lambda_{r'}>\frac{\lambda \sum_{i=1}^{r'}{\lambda_i}}{r+\lambda r'}=\frac{\lambda r' \kappa_{r'}}{r+\lambda r'}.$$
Then it follows from definition of $\rho$ in Theorem~\ref{thm:sym_global} that $r'\leq \rho$. We provide a case by case analysis for the above partition here.

\paragraph{Case 1. [$\cE=[r'], \ r'=\rho$]} When $\W$ corresponds to the top-$\rho$ eigenvectors of $\M$, we retrieve the global optimal solution described by Theorem~\ref{thm:sym_global}. Therefore, all such critical points are global minima.

\paragraph{Case 2. [$\cE=[r'], \ r'<\rho$]} Let $\W_r:=[\W,\W_\perp]$ be the top-$r$ eigenvectors of $\M$ and $\V_\perp$ span the orthogonal subspace of $\V$, i.e. $\V_r:=[\V,\V_\perp]$ be an orthonormal basis for $\R^r$. Define $\U(t)=\W_r\Sigma'\V_r^\top$ where $\sigma'_i=\sqrt{\sigma_i^2+t^2}$ for $i\leq r$. Observe that $$\U(t)^\top\U(t)=\V\Sigma\V^\top + t^2\V_r^\top\V_r = \U^\top\U+t^2\I_r$$ so that for all $t$, the parametric curve $\U(t)$ is equalized. The value of the loss function at $\U(t)$ is given by:
\begin{align*}
\ell(\U(t))&=\sum_{i=1}^{r}{(\lambda_i - \sigma_i^2 - t^2)^2}+\sum_{i=r+1}^{d}{(\lambda_i)^2}\\
&=\ell(\U)+rt^4-2t^2\sum_{i=1}^{r}{(\lambda_i - \sigma_i^2 )}.
\end{align*}
Furthermore, since $\U(t)$ is equalized, we obtain the following form for the regularizer:
\begin{align*}
R(\U(t))&=\frac{\lambda}{r}\|\U(t)\|_F^4
=\frac{\lambda}{r}\left(\|\U\|_F^2 + rt^2\right)^2\\
&=\ell(\U)+\lambda r t^4 + 2 \lambda t^2 \|\U\|_F^2.
\end{align*}
Now define $g(t):=\ell(\U(t))+R(\U(t))$ and observe
\begin{align*}
g(t)&=\ell(\U)+R(\U)+rt^4-2t^2\sum_{i=1}^{r}{(\lambda_i - \sigma_i^2 )}\\
&+\lambda r t^4 + 2 \lambda t^2 \|\U\|_F^2.
\end{align*}
It is easy to verify that $g'(0)=0$. Moreover, the second derivative of $g$ at the origin is given as:
\begin{align*}
g''(0)&=-4\sum_{i=1}^{r}{(\lambda_i - \sigma_i^2 )}+ 4 \lambda \|\U\|_F^2\\
&=-4\sum_{i=1}^{r}{\lambda_i }+ 4 (1+\lambda) \|\U\|_F^2\\
&=-4\sum_{i=1}^{r}{\lambda_i }+ 4 \frac{r+ r\lambda}{r+\lambda r'}\sum_{i=1}^{r'}{\lambda_i}
\end{align*}
where the last equality follows from the fact Equation~\eqref{eq:sol} and the fact that $\|\U\|_F^2 =\sum_{i=1}^{r'}{\sigma_i^2}$. To get a sufficient condition for $\U$ to be a strict saddle point, we set $g''(0)<0$:
\begin{align*}
&-4\sum_{i=r'+1}^{r}{\lambda_i }+ 4 \frac{(r-r')\lambda}{r+\lambda r'}\sum_{i=1}^{r'}{\lambda_i} < 0\\
&\implies \frac{(r-r')\lambda}{r+\lambda r'}\sum_{i=1}^{r'}{\lambda_i} < \sum_{i=r'+1}^{r}{\lambda_i }\\
&\implies \lambda< \frac{(r+\lambda r')\sum_{i=r'+1}^{r}{\lambda_i }}{(r-r')\sum_{i=1}^{r'}{\lambda_i} }\\
&\implies \lambda(1-\frac{ r'\sum_{i=r'+1}^{r}{\lambda_i }}{(r-r')\sum_{i=1}^{r'}{\lambda_i} })<\frac{r\sum_{i=r'+1}^{r}{\lambda_i }}{(r-r')\sum_{i=1}^{r'}{\lambda_i} }\\
&\implies \lambda<\frac{r\sum_{i=r'+1}^{r}{\lambda_i }}{(r-r')\sum_{i=1}^{r'}{\lambda_i} -  r'\sum_{i=r'+1}^{r}{\lambda_i }}\\
&\implies \lambda<\frac{r h(r')}{\sum_{i=1}^{r'}{(\lambda_i - h(r'))}}
\end{align*}
where $h(r'):=\frac{\sum_{i=r'+1}^{r}{\lambda_i }}{r-r'}$ is the average of the eigenvalues $\lambda_{r'+1},\cdots,\lambda_{r}$. It is easy to see that the right hand side is monotonically decreasing with $r'$, since $h(r')$ monotonically decrease with $r'$. Hence, it suffices to make sure that $\lambda$ is smaller than the right hand side for the choice of $r'=r -1$, i.e. $\lambda<\frac{r\lambda_r}{\sum_{i=1}^{r}(\lambda_i-\lambda_r)}$.

\paragraph{Case 3. [$\cE\neq[r']$]} We show that all such critical points are strict saddle points. Let $\w'$ be one of the top $r'$ eigenvectors that are missing in $\W$. Let $j\in\cE$ be such that $\w_j$ is not among the top $r'$ eigenvectors of $\M$. For any $t\in [0,1]$, let $\W(t)$ be identical to $\W$ in all the columns but the $j^{\text{th}}$ one, where $\w_j(t)=\sqrt{1-t^2}\w_j+t\w'$. Note that $\W(t)$ is still an orthogonal matrix for all values of $t$. Define the parametrized curve $\U(t):=\W(t)\Sigma\V^\top$ for $t\in[0,1]$ and observe that:
\begin{align*}
\| \U-\U(t) \|_F^2 &= \sigma_j^2\|\w_j - \w_j(t)\|^2\\
&= 2\sigma_j^2(1-\sqrt{1-t^2}) \leq t^2 \tr{\M}
\end{align*}
That is, for any $\epsilon>0$, there exist a $t>0$ such that $\U(t)$ belongs to the $\epsilon$-ball around $\U$. We show that $f(\U(t))$ is strictly smaller than $f(\U)$, which means $\U$ cannot be a local minimum. Note that this construction of $\U(t)$ guarantees that $R(\U')=R(\U)$. In particular, it is easy to see that $\U(t)^\top\U(t)=\U^\top\U$, so that $\U(t)$ remains equalized for all values of $t$. Moreover, we have that 
\begin{align*}
&f(\U(t))-f(\U)= \| \M-\U(t)\U(t)^\top \|_F^2 - \| \M-\U\U^\top \|_F^2\\
&= - 2\tr(\Sigma^2\W(t)^\top\M\W(t)) + 2\tr(\Sigma^2\W^\top\M\W)\\
&=-2\sigma^2_j t^2(\w_j(t)^\top\M\w_j(t)-\w_j^\top\M\w_j) <0,
\end{align*}
where the last inequality follows because by construction $\w_j(t)^\top\M\w_j(t) > \w_j^\top\M\w_j$. Define $g(t):=f(\U(t))=\ell(\U(t))+R(\U(t))$. To see that such saddle points are non-degenerate, it suffices to show $g''(0)<0$. It is easy to check that the second directional derivative at the origin is given by $$g''(0)=-4\sigma^2_j(\w_j(t)^\top\M\w_j(t)-\w_j^\top\M\w_j)<0,$$ which completes the proof.
\end{proof}

\begin{proof}[Proof of Lemma~\ref{lem:local}]
Follows from Lemma~\ref{lem:minima_eqz}
\end{proof}

\begin{proof}[Proof of Theorem~\ref{thm:main_geometry}]
Follows from Lemma~\ref{lem:minima_eqz} and Lemma~\ref{lem:no_spurious}.
\end{proof}